\documentclass[]{scholar7}

\usepackage{graphicx}%
\usepackage{multirow}%
\usepackage{amsmath,amssymb,amsfonts,bm}%
\usepackage{amsthm}%
\usepackage{mathrsfs}%
\usepackage[title]{appendix}%
\usepackage{xcolor}%
\usepackage{textcomp}%
\usepackage{manyfoot}%
\usepackage{booktabs}%
\usepackage{algorithm}%
\usepackage{algorithmicx}%
\usepackage{algpseudocode}%
\usepackage{listings}%
\usepackage{natbib}
\usepackage{hyperref}
\usepackage{times}
\usepackage{subcaption}
\usepackage{placeins}
\usepackage{fontawesome}

\usepackage{titletoc}
\usepackage[subfigure]{tocloft}
\newcommand{\Xcal}{\mathcal{X}}
\newcommand{\Ycal}{\mathcal{Y}}

\newcommand{\ThetaSpace}{\Theta}

\newcommand{\Prob}{\mathbb{P}}
\newcommand{\wb}{w}
\newcommand{\xb}{x}
\newcommand{\Phib}{\Phi}
\newcommand{\Xb}{\mathsf{X}}

\newcommand{\Ccalctx}{\mathcal{C}_{\text{ctx}}}
\newcommand{\Ctx}{C}
\newcommand{\Ccal}{\mathcal{C}}

\newcommand{\Peers}{\mathcal{J}_{-t}}
\newcommand{\Simplex}{\Delta^{|\Peers|-1}}
\newcommand{\Pdesign}{\mathbb{P}_{\text{design}}}
\newcommand{\Pfit}{\mathbb{P}_{\text{fit}}}
\newcommand{\Peval}{\mathbb{P}_{\text{eval}}}
\newcommand{\PIER}{R_t}
\newcommand{\hPIER}{\widehat{R}_t}
\newcommand{\Uniqueness}{\mathcal{U}_t{}}
\newcommand{\hUniqueness}{\widehat{\mathcal{U}}_t{}}

\newcommand{\Uconv}{\mathcal{U}_t^{\mathrm{conv}}}
\newcommand{\Ulin}{\mathcal{U}_t^{\mathrm{lin}}}
\newcommand{\Uker}{\mathcal{U}_t^{\mathrm{ker}}}

\theoremstyle{thmstyleone}%
\newtheorem{theorem}{Theorem}%  meant for continuous numbers
\newtheorem{proposition}[theorem]{Proposition}% 

\theoremstyle{thmstyletwo}%

\theoremstyle{thmstylethree}%

\title{Quantifying Model Uniqueness in Heterogeneous AI Ecosystems}

\author[1]{Lei You}

\affiliation[1]{Department of Engineering Technology, Technical University of Denmark, Lautrupvang 15, 2750 Ballerup, Capital Region, Denmark}

\correspondence{\email{leiyo@dtu.dk}}
\date{February 2026}

\abstract{%
As AI systems evolve from isolated predictors into complex, heterogeneous ecosystems of foundation models and specialized adapters, distinguishing genuine behavioral novelty from functional redundancy becomes a critical governance challenge. Here, we introduce a statistical framework for auditing model uniqueness based on In-Silico Quasi-Experimental Design (ISQED). By enforcing matched interventions across models, we isolate intrinsic model identity and quantify uniqueness as the Peer-Inexpressible Residual (PIER), i.e. the component of a target's behavior strictly irreducible to any stochastic convex combination of its peers, with vanishing PIER characterizing when such a routing-based substitution becomes possible. We establish the theoretical foundations of ecosystem auditing through three key contributions. First, we prove a fundamental limitation of observational logs: uniqueness is mathematically non-identifiable without intervention control. Second, we derive a scaling law for active auditing, showing that our adaptive query protocol achieves minimax-optimal sample efficiency ($d\sigma^2\gamma^{-2}\log(Nd/\delta)$). Third, we demonstrate that cooperative game-theoretic methods, such as Shapley values, fundamentally fail to detect redundancy. We implement this framework via the DISCO (Design-Integrated Synthetic Control) estimator and deploy it across diverse ecosystems, including computer vision models (ResNet/ConvNeXt/ViT), large language models (BERT/RoBERTa), and city-scale traffic forecasters. These results move trustworthy AI beyond explaining single models: they establish a principled, intervention-based science of auditing and governing heterogeneous model ecosystems. The code is available on 
\faGithub~\url{https://github.com/youlei202/ISQED}.
}

\begin{document}

\maketitle
\section*{Introduction}

Machine learning systems that affect users and institutions are rarely deployed in isolation. A content-moderation pipeline may combine pre-filters, topic classifiers, ranking models and large language models (LLMs)~\cite{bommasani2021opportunities}. A recommendation platform may maintain multiple recommender models for different device classes, regions and latency budgets. As models are added, adapted and retired over time, the ecosystem can accumulate redundancy: distinct artefacts that behave identically in practice, or whose behaviour could be reproduced at negligible cost by composing existing components. Such ecosystems are increasingly shaped by foundation-model pipelines built on Transformer architectures~\cite{vaswani2017attention} and expanded via fine-tuning, distillation~\cite{hinton2015distilling} and parameter-efficient adapters~\cite{houlsby2019parameter}.

Governance questions arise naturally in this setting. A regulator or platform operator may want to know whether a proposed target model contributes genuinely new behaviour, incurs additional risk, or is effectively a reparameterisation of existing models. A cloud provider may wish to decommission models that are functionally redundant to save energy or reduce maintenance cost~\cite{strubell2019energy}. A user organisation negotiating licences may need to demonstrate that a model it deploys has unique value relative to open-source peers~\cite{tramer2016stealing,uchida2017embedding,rouhani2019deepsigns}. These questions concern the position of a given model \emph{inside an ecosystem} rather than the internal logic of that model.

Explainable AI (XAI) has primarily addressed single-model questions. Feature attribution methods quantify which input dimensions are important for a specific model (e.g., LIME~\cite{ribeiro2016why}, Integrated Gradients~\cite{sundararajan2017axiomatic}, and SHAP~\cite{lundberg2017unified}). Counterfactual explanations describe which perturbations would change a prediction~\cite{wachter2018counterfactual}. Representation analysis probes internal layers of neural networks (e.g., SVCCA~\cite{raghu2017svcca} and CKA~\cite{kornblith2019similarity}). These tools are valuable for understanding an individual model in isolation. They do not, however, measure whether that model is replaceable by other models in the same ecosystem. In multi-model settings, cooperative game-theoretic methods such as Shapley value attribution~\cite{shapley1953value} have been proposed to assign credit among models, but these are designed to share rewards fairly, not to detect redundancy. Related work on combining multiple trained models includes weight-averaging approaches such as model soups~\cite{wortsman2022model}.

%TC:ignore
\begin{figure*}[t]
    \centering
    \includegraphics[width=\linewidth]{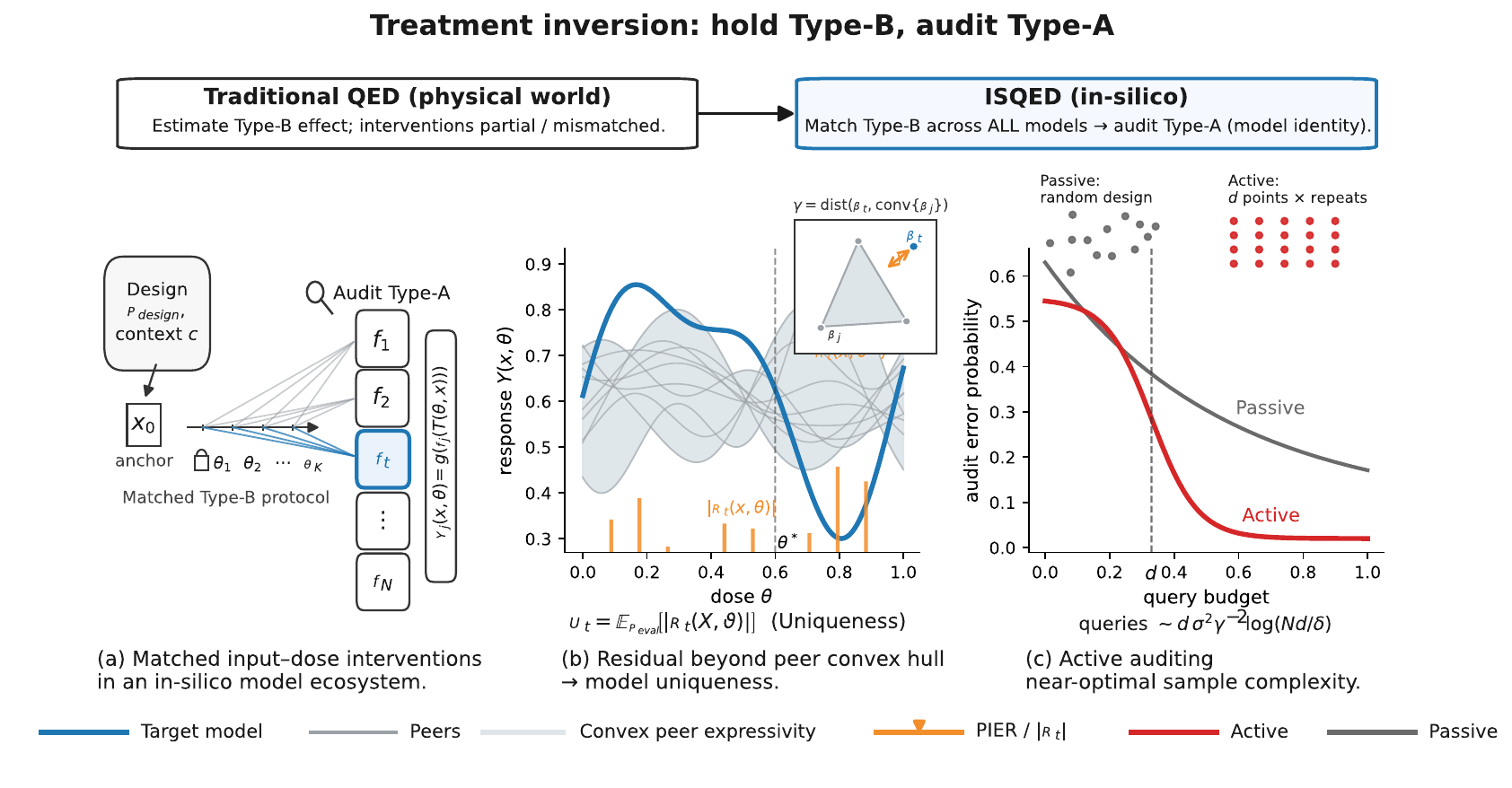}
\caption{\textbf{Treatment inversion: hold Type-B fixed, audit Type-A (model identity).}
    \textbf{a}, Matched input--dose interventions in an in-silico ecosystem: an auditor selects an anchor input and a Type-B intervention track (doses $\theta_1,\dots,\theta_K$) and queries all models on the matched pairs to isolate Type-A (model-identity) differences.
    \textbf{b}, Uniqueness as a residual beyond the convex peer hull. The target response $Y_t(x,\theta)$ is projected onto the convex hull of peer responses, yielding the peer-inexpressible residual $\PIER(x,\theta)=R_t(x,\theta)$; its evaluation-average magnitude $\Uniqueness$ summarizes behavioral novelty.
    \textbf{c}, Active auditing. Because the auditor controls queries in-silico, adaptive designs can detect non-zero uniqueness with far fewer queries than passive sampling; in a local linear model the minimax-optimal scaling is $d\sigma^2\gamma^{-2}\log(Nd/\delta)$.}
    \label{fig:concept}
\end{figure*}
%TC:endignore

A central difficulty is to \emph{separate two sources of variation} in responses. The first stems from the external conditions under which models are evaluated, for example the distribution of user queries or the strength of stress-test perturbations. The second stems from the intrinsic identity of the model: its architecture, training data, objective and optimisation. In the language of causal inference~\cite{pearl2009causality,imbens2015causal}, the former corresponds to a conventional intervention applied to model inputs, which we refer to as a ``Type-B'' treatment, while the latter reflects a unit-level effect associated with intrinsic model identity, which we treat as a ``Type-A'' factor. In observational domains such as policy evaluation, quasi-experimental designs~\cite{campbell1963experimental,angrist2009mostly} attempt to estimate the effect of a Type-B intervention on units that cannot all be manipulated, for example, when only one state changes its law. Here we are interested in the opposite quantity: the effect of the Type-A treatment, conditional on a controlled pattern of Type-B interventions, in an in-silico environment where we can manipulate all models.

In an in-silico ecosystem we can query every model on matched input–dose pairs and design synthetic interventions on inputs. This allows us to impose a common Type-B protocol across the ecosystem and treat differences in responses as arising from Type-A alone. The central idea of this work is to formalise this setting as an \emph{in-silico quasi-experimental design (ISQED)} framework (\figurename~\ref{fig:concept}) and to turn the question \emph{``is the target unique relative to its peers?''} into a well-posed statistical problem. Within a fixed context we define a convex peer expressivity class: the family of convex combinations of peer responses in a Hilbert space of scalarised outputs under a fitting design. The population \emph{peer-inexpressible residual}, abbreviated \emph{PIER}, is the orthogonal residual of the target's response after projection onto this class. A population uniqueness functional is obtained by integrating the magnitude of PIER under a possibly distinct evaluation design. These quantities capture precisely the part of the target’s behaviour that cannot be reproduced by simple stochastic mixtures of peers under the chosen designs, and therefore cannot be substituted by a fixed routing over the ecosystem.

Our proposed empirical \emph{design-integrated synthetic control (DISCO) estimator} implements this construction using finite samples, building on the synthetic control literature~\cite{abadie2010synthetic}. A fitting sample is used to estimate projection weights subject to a simplex constraint, and an evaluation sample is used to estimate PIER values and the uniqueness functional. Under mild assumptions, DISCO has the usual statistical properties expected of a measurement instrument: the estimated projection weights converge to their population counterparts, the empirical uniqueness functional converges to the population functional, and its fluctuations admit an asymptotic normal approximation. A finite-sample design error bound separates estimation error from the contribution of finite anchor radius and dose-grid resolution. These \emph{passive} guarantees mean that, when used with fixed designs, DISCO behaves as a calibrated measurement of model uniqueness.

The in-silico nature of the ecosystem provides more than identifiability advantages. Because we can choose where to query models, we can ask how to allocate a limited number of queries to detect uniqueness as quickly as possible~\cite{settles2009active,pukelsheim2006optimal}. We analyse this question in a local linear structural model in which all models share a known feature map and differ only in their coefficients. In a noiseless setting, a small number of carefully chosen design points suffice to recover all coefficients exactly, and uniqueness reduces to a convex hull membership test in coefficient space. In a noisy setting with sub-Gaussian perturbations, repeated queries at these design points yield an \emph{active auditing} procedure whose sample complexity scales as $d \sigma^2 \gamma^{-2} \log(Nd/\delta)$, where $d$ is feature dimension, $\sigma^2$ is noise level, $\gamma$ is a uniqueness margin, $N$ the number of models, and $\delta$ is the target error probability. A minimax lower bound and a detection-versus-estimation result show that this dependence on $\sigma^2$, $\gamma$ and $\delta$ is optimal up to constants: in the worst case, detecting non-zero uniqueness cannot be easier than estimating the corresponding local-linear parameter. This reveals a fundamental scaling law for in-silico uniqueness audits.

Beyond rate statements, structural properties of the uniqueness functional matter for ecosystem governance. Because the projection is taken onto a convex hull, the resulting PIER is monotone in the peer set: adding peers can only decrease uniqueness. The convex hull is also conservative relative to richer peer expressivity classes such as linear spans or kernel-induced function spaces: uniqueness defined via convex combinations provides an upper bound on uniqueness defined via any nested richer class. The projection weights admit an interpretation as a stochastic routing policy over peers. If convex PIER is close to zero, the target can be replaced by routing queries to peers according to these weights without changing the expected response under the chosen designs. Conversely, if convex PIER remains large, even such simple mixtures fail to explain the target. These properties align convex PIER with the notion of actionable redundancy.

We further show that popular model-level attribution methods are not suited to this task. In particular, a model can be fully redundant in the sense of lying exactly in the convex hull of its peers, hence having zero PIER and zero uniqueness, and nevertheless receive strictly positive Shapley value. This illustrates that cooperative-game attribution is designed for sharing rewards fairly among contributors, not for identifying units that can be removed without loss. We also prove that, when models are only observed under unmatched observational designs, the population uniqueness functional with respect to a common reference design is not identifiable, even given arbitrarily large logs. In other words, the ability to impose matched interventions among models is not merely convenient but mathematically necessary for uniqueness auditing in the sense studied here.

The remainder of the main text is organised as follows. The Results section introduces the ISQED framework, defines PIER and model uniqueness, establishes passive and active theoretical guarantees, examines key structural and conceptual properties, and empirically audits heterogeneous ecosystems across language, vision, and spatiotemporal prediction tasks. The Methods section formalises the constructions and describes the DISCO estimator and active auditing protocols. Detailed proofs and additional technical material are provided in the Supplementary Information.

\section*{Results}

\subsection*{ISQED Framework for AI Ecosystems}

\textbf{Notation.}
We use uppercase symbols for random variables and lowercase for their realisations; in particular, $(X,\vartheta)\sim\Pfit$ or $\Peval$, while $(x,\theta)$ denotes a fixed input--dose pair.
We reserve $\theta$ for a fixed dose value and use $\vartheta$ for a random dose drawn from a design distribution.
We do not use boldface: vectors and matrices are identified by their declared domains (e.g., $w \in \Delta^{N-1}$, $\phi(x,\theta)\in\R^d$, $\Phi\in\R^{d\times d}$, $\Xb_{-t}^{\mathrm{fit}}\in\R^{m\times|\Peers|}$).
Throughout, $\PIER(x,\theta)=R_t(x,\theta)$ denotes the (signed) peer-inexpressible residual, and the scalar uniqueness score is $\Uniqueness=\E_{\Peval}[|\PIER(X,\vartheta)|]$; when we plot a single ``PIER'' value, we report an aggregate residual magnitude (typically a mean absolute residual).

We consider an input space $\Xcal$ and a dose space $\ThetaSpace$. A deterministic intervention map $T : \ThetaSpace \times \Xcal \to \Xcal$ perturbs an input $\xb$ according to a dose parameter $\theta$, for example by scaling, masking or otherwise modifying selected features. A model ecosystem consists of functions $f_j : \Xcal \to \Ycal$ indexed by $j \in \mathcal{J} = \{1,\dots,N\}$. A scalarisation map $g : \Ycal \to \R$ converts model outputs into real-valued scores. The scalarised response of model $j$ at an input–dose pair $(x,\theta)$ is
\[
Y_j(\xb,\theta) = g\big(f_j(T(\theta,\xb))\big) \in \R.
\]

The ecosystem is evaluated within a context $\Ctx$ taking values in a space $\Ccalctx$. Intuitively, the context encodes a local region of the input–dose space, such as a neighbourhood around an anchor input combined with an intervention track. A global design distribution $\Pdesign$ on $\Xcal \times \ThetaSpace \times \Ccalctx$ describes how inputs, doses and contexts are generated. Conditioning on a fixed context value $c \in \Ccalctx$ yields a design distribution on $\Xcal \times \ThetaSpace$. Throughout the results section we fix such a context and suppress it in the notation.

Within the fixed context, we distinguish three roles for design distributions. A fitting design $\Pfit$ determines where responses are observed for learning a convex peer surrogate that best approximates the target's behaviour. An evaluation design $\Peval$ determines where PIER and uniqueness are assessed. Both $\Pfit$ and $\Peval$ are typically derived from $\Pdesign$, which represents the population of interest and provides a conceptual reference, for example the distribution of real user queries. The in-silico setting allows all three to be chosen by the auditor, subject to computational constraints. An honesty condition requires that fitting and evaluation samples are drawn independently from $\Pfit$ and $\Peval$ respectively, and that the fitting sample is not reused for evaluation. This separates the process of learning a convex peer surrogate from the process of auditing uniqueness.

A target index $t \in \mathcal{J}$ is fixed and its peers are $\Peers = \mathcal{J}\setminus \{t\}$. The vector of peer responses at $(x,\theta)$ is
\[
\Phib_{-t}(\xb,\theta) = (Y_j(\xb,\theta))_{j \in \Peers} \in \R^{|\Peers|}.
\]
Convex combinations of these responses form the basic building blocks of the peer surrogate.

\subsection*{A Convex Notion of PIER}

The core population quantities are defined in a Hilbert space. For the fixed context and fitting design $\Pfit$, consider the space $L^2(\Pfit)$ of measurable functions $h : \Xcal \times \ThetaSpace \to \R$ with finite second moment under $\Pfit$. The inner product is
\[
\langle h_1,h_2 \rangle = \E_{\Pfit}\big[h_1(X,\vartheta) h_2(X,\vartheta)\big],
\]
with induced norm $\|h\|_2 = \sqrt{\langle h,h \rangle}$.

For a weight vector $\wb \in \R^{|\Peers|}$ and an input–dose pair $(x,\theta)$, the convex combination of peer responses is $h_{\wb}(x,\theta) = \wb^\top \Phib_{-t}(x,\theta)$. The convex peer expressivity class is the set of such functions with weights in the simplex,
\[
\Ccal = \{h_{\wb} : \wb \in \Simplex\} \subset L^2(\Pfit),
\]
where $\Simplex$ is the set of non-negative vectors in $\R^{|\Peers|}$ that sum to one. This class captures \emph{all behaviours that can be realised by stochastic routing} over peers with a fixed context-independent routing distribution $\wb$.

The squared $L^2(\Pfit)$ distance between the target response and a peer mixture is
\[
L(\wb) = \E_{\Pfit}\big[(Y_t(X,\vartheta) - \wb^\top \Phib_{-t}(X,\vartheta))^2\big].
\]
Under a mild identifiability condition, this residual has a unique minimiser $\wb^*$ in the relative interior of the simplex, and the residual is locally strongly convex around $\wb^*$. The function $h_{\wb^*}$ is the population projection of $Y_t$ onto the closure of $\Ccal$ in $L^2(\Pfit)$. The population PIER is defined as the corresponding residual,
\[
\PIER(x,\theta) = Y_t(x,\theta) - \wb^{*\top} \Phib_{-t}(x,\theta).
\]
If $\PIER$ is almost surely zero under $\Pfit$, then the target behaves as a convex mixture of peers under the fitting design. If $\PIER$ is not almost surely zero, then the target exhibits behaviour that cannot be expressed by such mixtures.

To aggregate PIER into a scalar measure, an evaluation design $\Peval$ on $\Xcal \times \ThetaSpace$ is introduced. Under mild integrability conditions, the \emph{population uniqueness functional} is defined as
\[
\Uniqueness = \E_{\Peval}\big[|\PIER(X,\vartheta)|\big].
\]
This functional depends on the fitting design through $\wb^*$, on the evaluation design and on the peer set. It may be interpreted as an average magnitude of peer-inexpressible behaviour along the dose–response trajectory specified by $\Peval$.

The convex formulation has a useful probabilistic interpretation. The weights $\wb^*$ define a stochastic routing policy that sends a query to peer $j$ with probability $w^*_j$ and returns its scalarised response. The conditional expectation of this routed response is exactly $h_{\wb^*}(x,\theta)$. When $\PIER$ is small in $L^2(\Pfit)$ and $\Uniqueness$ is small under $\Peval$, the target can be substituted, in expectation, by randomised routing over peers according to $\wb^*$ under both designs. When $\PIER$ is exactly zero, such routing is indistinguishable from the target along the PIER metric.

\subsection*{Passive Statistical Guarantees for DISCO}

We now propose the DISCO estimator that implements the population construction using finite samples. For a fixed context, a fitting sample $(X_i^{\mathrm{fit}},\vartheta_i^{\mathrm{fit}})_{i=1}^m$ is drawn independently from $\Pfit$ and used to compute responses $Y_{j,i}^{\mathrm{fit}} = Y_j(X_i^{\mathrm{fit}},\vartheta_i^{\mathrm{fit}})$ for all models. The peer response vectors are stacked into a design matrix $\Xb_{-t}^{\mathrm{fit}} \in \R^{m \times |\Peers|}$ and the target responses into a vector $y_t^{\mathrm{fit}} \in \R^m$. The empirical projection weights are obtained by solving a regularised least-squares problem over the simplex,
\[
\widehat{\wb} \in \argmin_{\wb \in \Simplex} \left\{ \frac{1}{m} \sum_{i=1}^m \big(Y_{t,i}^{\mathrm{fit}} - \wb^\top \Phib_{-t,i}^{\mathrm{fit}}\big)^2 + \lambda_m \|\wb\|_2^2 \right\},
\]
with a regularisation parameter $\lambda_m$ that tends to zero as $m$ increases.

An independent evaluation sample $(X_i^{\mathrm{eval}},\vartheta_i^{\mathrm{eval}})_{i=1}^n$ is drawn from $\Peval$. The empirical PIER at the $i$-th evaluation point is
\[
\hPIER{}_i = Y_t(X_i^{\mathrm{eval}},\vartheta_i^{\mathrm{eval}}) - \widehat{\wb}^\top \Phib_{-t}(X_i^{\mathrm{eval}},\vartheta_i^{\mathrm{eval}}),
\]
and the empirical uniqueness functional is the average absolute PIER,
\[
\hUniqueness = \frac{1}{n} \sum_{i=1}^n |\hPIER{}_i|.
\]

Under mild regularity, honesty and identifiability conditions, the DISCO estimator behaves as expected of a sound statistical instrument. Appendix~\ref{app:consistency} shows that the empirical residual converges uniformly to the population residual on the simplex, the regularisation term vanishes, and the argmin theorem for $M$-estimators implies that $\widehat{\wb}$ converges in probability to the population projection weights $\wb^*$. A continuous mapping argument then yields convergence in probability of $\hUniqueness$ to $\Uniqueness$. Appendix~\ref{app:normality} shows that under additional smoothness assumptions on the dependence of $|Y_t - \wb^\top \Phib_{-t}|$ on $\wb$, a central limit theorem holds for $\hUniqueness$: when both $m$ and $n$ grow and their ratio tends to a finite positive constant, the scaled difference $\sqrt{n}(\hUniqueness - \Uniqueness)$ converges in distribution to a centred Gaussian law with finite variance. This justifies normal approximations and bootstrap-based confidence intervals for $\Uniqueness$.

A finite-sample design error bound clarifies the role of anchors and dose grids. In many practical setups, evaluation points are associated with nearby anchor points in the input–dose space, and responses at anchors are used as proxies for responses at evaluation points. Under Lipschitz continuity assumptions on the scalarised responses, Appendix~\ref{app:error_bound} shows that the error between the anchor-based empirical PIER and the population PIER at an evaluation point is bounded by a term proportional to $\|\widehat{\wb} - \wb^*\|_2$ (estimation error) plus a term proportional to the input distance between the evaluation point and its anchor and the dose-grid spacing (design error). This decomposition shows that improving either estimation accuracy or anchor/grid resolution decreases PIER error, and that the contributions of statistical and geometric design choices can be analysed separately.

Taken together, these results place PIER and DISCO on a standard statistical footing. The quantities that appear throughout the Results section can be interpreted as population parameters, estimators of these parameters and formal error controls, rather than as heuristic scores.

\subsection*{Active Auditing Achieves Optimal Sample Efficiency}

The passive theory treats $\Pfit$ and $\Peval$ as fixed reference designs. The in-silico setting allows a different regime in which the auditor actively chooses where to query models. To analyse the benefits of this control, we consider a local linear structural model within a fixed context. A feature map $\phi : \Xcal \times \ThetaSpace \to \R^d$ is assumed known, and each model $j$ is parametrised by a coefficient vector $\beta_j \in \R^d$ such that
\[
Y_j(x,\theta) = \phi(x,\theta)^\top \beta_j
\]
for all $(x,\theta)$ in the region of interest. The feature map may arise from hand-crafted basis functions, a representation learnt in a pretraining phase, or from linearisation of differentiable models via Jacobians or neural tangent kernels~\cite{jacot2018neural}. The target and peer behaviours are thus encoded in coefficient space. Remark that we do not assume that real-world ecosystems are globally linear. The local linear structural model is a stylised approximation, motivated by first-order Taylor expansions and shared-feature architectures, and is used to expose the best-case sample complexity of active in-silico auditing.

The peer coefficients $\{\beta_j : j \in \Peers\}$ form a finite set in $\R^d$, and their convex hull, denoted $\mathcal{H}_{\mathrm{peer}}$, represents all behaviours attainable by convex mixtures of peers in coefficient space. The target coefficient is $\beta_t$. The distance between $\beta_t$ and $\mathcal{H}_{\mathrm{peer}}$ defines a uniqueness margin $\gamma$, which is zero if the target lies in the convex hull and strictly positive otherwise. The question of uniqueness reduces to deciding whether $\gamma$ is zero or positive using as few queries as possible.

In a noiseless setting, the auditor can exploit the linear structure fully. By choosing $d$ design points $(x_\ell,\theta_\ell)$ such that the resulting feature matrix
\[
\Phi = \begin{bmatrix}
\phi(x_1,\theta_1)^\top\\
\vdots\\
\phi(x_d,\theta_d)^\top
\end{bmatrix}
\]
is invertible, the auditor can recover each coefficient vector exactly. Evaluating model $j$ at these $d$ points yields a response vector $y_j = (Y_j(x_1,\theta_1),\dots,Y_j(x_d,\theta_d))^\top$. Solving the linear system $\Phi \beta_j = y_j$ gives $\beta_j = \Phi^{-1}y_j$. This holds for all models simultaneously, and the uniqueness question becomes a convex feasibility problem in coefficient space: the target is non-unique if and only if its coefficient vector lies in the convex hull of peer coefficients. Appendix~\ref{app:uniqueness_detection} formalises this argument and shows that exactly $d$ noiseless queries per model suffice for an exact uniqueness decision in the linear structural model.

In a noisy setting, model evaluations are perturbed by randomness. For example, stochastic components in the system or approximate inference may introduce zero-mean noise with variance proxy $\sigma^2$. The auditor can still query each model at a fixed set of $d$ design points, but now multiple repetitions are needed to average out noise. Under sub-Gaussian noise assumptions, repeated queries at each design point and a simple averaging-and-inversion scheme yield coefficient estimates $\widehat{\beta}_j$ for all models. The empirical distance from the target estimate to the convex hull of peer estimates is then compared to a threshold. Appendix~\ref{app:sample_complexity} shows that, if the uniqueness margin $\gamma$ is strictly positive, then there exists a constant $C$ depending only on the feature matrix such that
\[
r \geq C\,\frac{\sigma^2}{\gamma^2} \log\!\left(\frac{Nd}{\delta}\right)
\]
repetitions per design point suffice to distinguish $\gamma = 0$ from $\gamma > 0$ with error probability at most $\delta$. The total number of noisy queries per model is $d r$, which yields a sample-complexity bound of order $d \sigma^2 \gamma^{-2} \log(Nd/\delta)$. Larger uniqueness margins permit fewer queries, and the dependence on the number of models and the target error probability is only logarithmic. To validate this scaling law, we simulated audits in a controlled linear structural environment. As shown in \figurename~\ref{fig:auditing}A, the active auditing strategy (red) achieves a sharp phase transition, driving the error rate to near-zero once the query budget exceeds the feature dimension $d$. In contrast, the passive strategy (blue) converges significantly slower, requiring an order of magnitude more queries to reach comparable reliability. The empirical convergence rate of the active strategy perfectly matches our theoretical prediction (dashed arrow), confirming the efficiency gain of the in-silico design.

%TC:ignore
\begin{figure}[t]
    \centering

    \begin{subfigure}{0.32\linewidth}
        \centering
        \includegraphics[width=\linewidth]{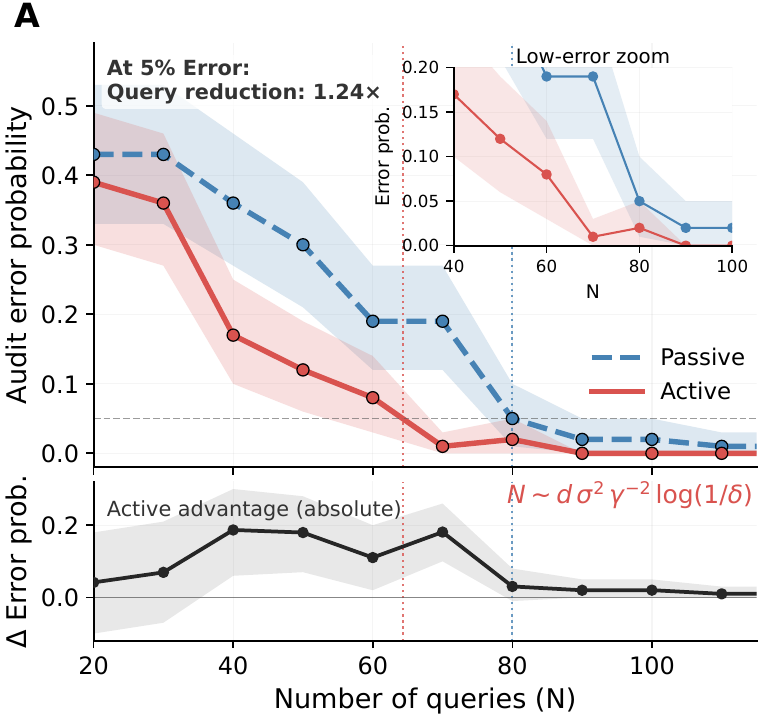}
        % \caption{PIER heatmap across targets and doses}
        \label{fig:fig2a}
    \end{subfigure}
    \hfill
    \begin{subfigure}{0.32\linewidth}
        \centering
        \includegraphics[width=\linewidth]{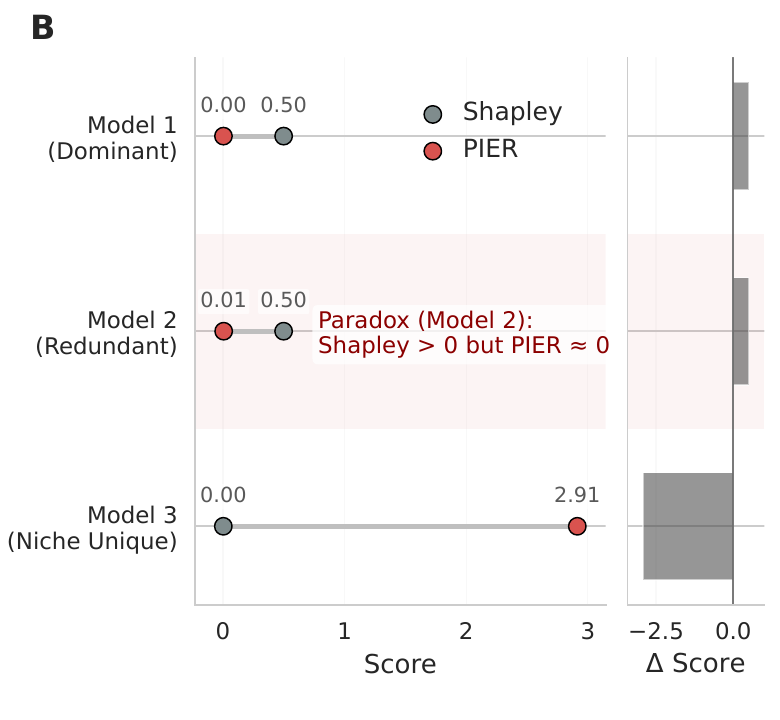}
        % \caption{Dose--response curves of PIER}
        \label{fig:fig2b}
    \end{subfigure}
   \hfill
    \begin{subfigure}{0.32\linewidth}
        \centering
        \includegraphics[width=\linewidth]{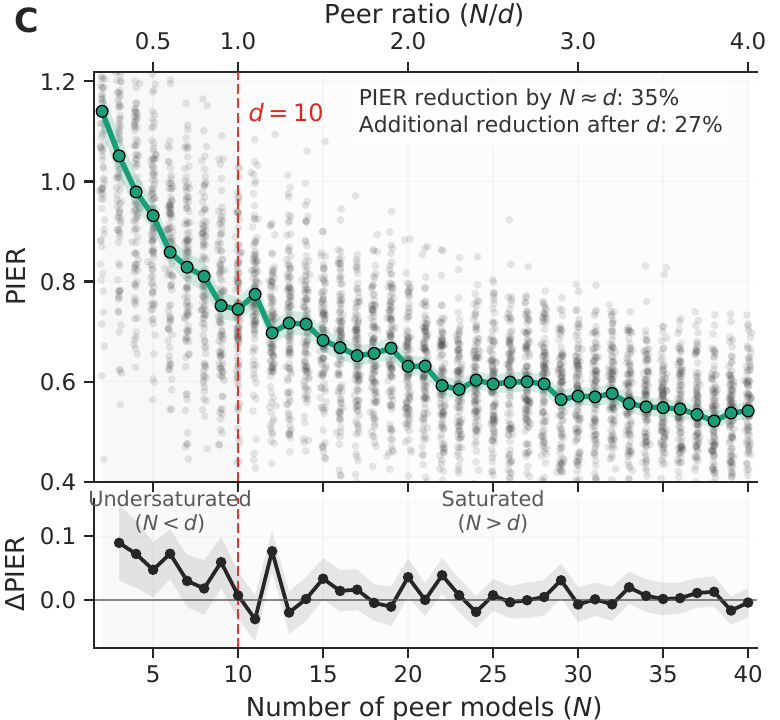}
        % \caption{DISCO routing structure vs.\ dose}
        \label{fig:fig2c}
    \end{subfigure}
\caption{\textbf{DISCO and PIER: auditing, attribution, and ecosystem effects.}
(A) Active auditing reduces query complexity compared to passive sampling, achieving a $1.24\times$ query reduction at 5\% audit error.
(B) Attribution versus auditing in a toy ecosystem: Shapley assigns positive value to a redundant model, while PIER assigns near-zero uniqueness and correctly identifies a niche, non-substitutable model.
(C) Ecosystem saturation: as the number of peer models $N$ grows beyond the intrinsic task dimension $d$ (here $d=10$), average uniqueness (mean absolute PIER) collapses, revealing diminishing returns from adding peers.}
\label{fig:auditing}
\end{figure}
%TC:endignore

Active auditing might appear to enjoy unique advantages over passive estimation, and one may wonder whether detection of uniqueness could be substantially easier than estimation of the corresponding structural parameter.  Appendix~\ref{app:minimax_lower_bound} establishes that this is not the case in the worst-case local linear model. In one dimension with constant features and a trivial peer set, the problem reduces to distinguishing whether a Gaussian mean parameter is zero or equal to $\gamma$. Classical information-theoretic inequalities yield a minimax lower bound of order $\sigma^2 \gamma^{-2} \log(1/\delta)$ noisy queries for any test that decides between these hypotheses with error probability at most $\delta$. At the same time, any estimator of the mean that aims to achieve accuracy of order $\gamma$ with confidence level $1-\delta$ requires the same order of queries. Based on this result, Appendix~\ref{app:detection_vs_estimation} shows that, in this regime, there is no rate separation between detection and estimation: any auditing procedure that reliably detects non-zero uniqueness must operate at essentially the same sample-complexity scale as one that estimates the underlying parameter.

These results identify a fundamental scaling law for in-silico uniqueness audits in a local linear regime. While the constants and precise dependence on $d$ and $N$ depend on the chosen feature map and design, the key message is that active control over design can dramatically reduce the dependence on the ambient input dimension and can achieve optimal dependence on noise level, uniqueness margin and desired error probability.

\subsection*{Conservatism, Monotonicity and Probabilistic Routing}

The convex formulation of the peer expressivity class has several structural properties that are useful for ecosystem governance. First, the projection weights admit an operational interpretation. Because $\wb^*$ lies in the simplex, it can be interpreted as a routing distribution over peers. If PIER is small, a system designer can replace the target by a stochastic router that, given any input–dose pair in the context, sends the query to peer $j$ with probability $w^*_j$ and returns its response. Under both $\Pfit$ and $\Peval$, the expected scalarised response of this router matches the convex combination $h_{\wb^*}$, and small PIER means that this expected response is close to that of the target. The uniqueness functional thus measures how well such a simple routing policy can substitute the target.

Second, uniqueness is monotone in the peer set. Consider two peer sets $\mathcal{J}_1 \subseteq \mathcal{J}_2$ and the corresponding convex expressivity classes $\Ccal(\mathcal{J}_1)$ and $\Ccal(\mathcal{J}_2)$. Enlarging the peer set enlarges the convex hull, and orthogonal projection onto a larger closed convex set in a Hilbert space cannot increase residual norm. The PIER defined with respect to $\mathcal{J}_2$ therefore has no larger $L^2(\Pfit)$ norm than the PIER defined with respect to $\mathcal{J}_1$. Boundedness of responses implies the same ordering for the uniqueness functional under $\Peval$. Appendix~\ref{app:monotonicity} formalises this argument. Empirically, we observe that this monotonic decay slows down significantly near the task's intrinsic dimension, marking a clear saturation threshold for the ecosystem (\figurename~\ref{fig:auditing}C). This characteristic provides a structural justification for greedy ecosystem pruning strategies in which peers are removed one by one based on their impact on the uniqueness of remaining models.

Third, convex PIER is conservative with respect to richer peer expressivity classes. In addition to $\Ccal_{\mathrm{conv}}$, one may consider the linear span of peer responses or the closure of a reproducing-kernel Hilbert space that contains them. Projections of $Y_t$ onto these larger classes yield alternative residuals and uniqueness functionals $\Ulin$ and $\Uker$. Because projection onto larger closed convex sets can only reduce residual norms, the corresponding uniqueness functionals satisfy
\[
\Uker \leq \Ulin \leq \Uconv.
\]
Appendix~\ref{app:monotonicity} proves this relationship. It follows that convex PIER provides an upper bound on uniqueness defined with respect to any nested richer class. This is desirable for auditing. If $\Uconv$ is small, then the target can be replaced by a very simple mechanism (a convex mixture or a stochastic router), which is an actionable notion of redundancy. If only a very rich non-linear peer class can explain the target, the cost of replacement may be comparable to training a new large model, and convex PIER will not declare the target redundant.

\subsection*{Uniqueness Versus Attribution and the Limits of Observational Logs}

The uniqueness functional studied here measures peer-inexpressible behaviour under a controlled in-silico design. It differs in aim and construction from model-level attribution methods based on cooperative game theory. In Shapley-based attribution~\cite{shapley1953value,lundberg2017unified}, models are treated as players in a cooperative game and a characteristic function assigns a performance score to each subset of models. The Shapley value of a model is its average marginal contribution to performance across all coalitions, subject to axioms of efficiency, symmetry and additivity. This notion of contribution is well suited to sharing rewards but does not encode redundancy.

A simple example illustrates the gap. Consider two models $M_1$ and $M_2$ with identical scalarised responses on all inputs and doses. Suppose that any system using either one alone achieves score $1$ and that using both together yields no further gain. Cooperative game theory assigns a Shapley value of $1/2$ to each model. From an auditing perspective, if $M_1$ is taken as target and $M_2$ as peer, the convex peer expressivity class contains $M_2$, the population PIER of $M_1$ is identically zero, and its uniqueness functional is zero. In other words, $M_1$ is fully redundant relative to $M_2$, yet its Shapley value is strictly positive. Appendix~\ref{app:shapley} makes this argument precise. This demonstrates that a positive Shapley value does not imply non-zero PIER and that Shapley-based attribution cannot serve as a sufficient statistic for redundancy detection.

We empirically validated this theoretical divergence in a controlled linear structural ecosystem (\figurename~\ref{fig:auditing}B). The experiment extends the theoretical setup to a realistic scenario involving a dominant model ($M_1$), its redundant clone ($M_2$), and a niche specialist ($M_3$). Consistent with the theoretical prediction, Shapley value assigns high credit to the redundant clone $M_2$ due to its performance contribution, masking its redundancy. In contrast, PIER correctly identifies $M_2$ as having vanishing uniqueness. Furthermore, the experiment reveals a converse failure mode: Shapley neglects the specialist $M_3$ due to low marginal utility, while PIER identifies it as the most geometrically unique component. This confirms that utility-based attribution and geometry-based auditing are orthogonal objectives.

The ability to impose matched interventions across models is also crucial. In many practical deployments, an auditor may only have access to historical logs in which each model has been queried under its own observational design, for example as part of different A/B tests or routing policies. In such a regime the joint distribution of $(X,\vartheta,Y_j(X,\vartheta))$ for each model $j$ is determined by an unknown design distribution $Q_j$ on $\Xcal \times \ThetaSpace$ and the deterministic response function $Y_j$. One might hope to reconstruct the population uniqueness functional with respect to a reference design $\Pdesign^\star$ from these logs. Appendix~\ref{app:non_identifiability} shows that this is impossible in general. There exist two ecosystems that induce identical joint distributions over observational logs for all models and all sample sizes, yet whose uniqueness functionals with respect to $\Pdesign^\star$ differ, with one equal to zero and the other strictly positive. Any procedure based purely on observational logs must therefore fail on at least one of these ecosystems. This non-identifiability result highlights that in-silico control over interventions is not merely a convenience but a structural prerequisite for uniqueness auditing in the sense studied here.

\subsection*{Uniqueness, Robustness and the Shape of the PIER Trajectory}

The uniqueness functional $\Uniqueness$ compresses the entire residual trajectory $(x,\theta) \mapsto \PIER(x,\theta)$ into a single scalar via an expectation under $\Peval$. This is a natural summary for many purposes but it inevitably loses information. In particular, a model may exhibit a large uniqueness score either because it consistently behaves differently from peers in a smooth and robust way, or because it reacts erratically to small perturbations, generating high-magnitude residuals only in narrow regions of the dose space.

Appendix~\ref{app:uniqueness_and_robustness} makes this ambiguity explicit. For any prescribed pair of constants $0 < L_{\mathrm{low}} < L_{\mathrm{high}}$, it constructs two targets that share the same peer set and the same uniqueness functional $\Uniqueness$ under a given $\Peval$, yet whose PIER functions have dramatically different regularity along the dose axis. In one case, the residual is constant in $\theta$ and hence perfectly stable, with Lipschitz constant bounded by $L_{\mathrm{low}}$. In the other, the residual oscillates rapidly with $\theta$ and has Lipschitz constant at least $L_{\mathrm{high}}$. The two models are indistinguishable at the level of the scalar uniqueness score but have very different robustness profiles. This shows that, without additional inductive bias, the magnitude of $\Uniqueness$ alone cannot separate ``capability-driven'' uniqueness from uniqueness arising from sensitivity to perturbations.

In practice, this suggests two complementary uses of the PIER framework. The scalar functional $\Uniqueness$ is appropriate for ranking models by their average degree of peer-inexpressible behaviour in a given context and for certifying strong forms of redundancy when it is close to zero. For more fine-grained auditing, especially when robustness to interventions matters, it is natural to inspect the full dose–response profile $\theta \mapsto \PIER(x,\theta)$ at representative inputs or to regularise the residual trajectory, for example by penalising its variation along $\theta$. The in-silico design, which treats $\theta$ as an explicit controlled coordinate, makes such analyses straightforward to implement.

\subsection*{Use Case Study: Auditing Large Language Models}

%TC:ignore
\begin{figure}[t]
    \centering

    \begin{subfigure}{0.32\linewidth}
        \centering
        \includegraphics[width=\linewidth]{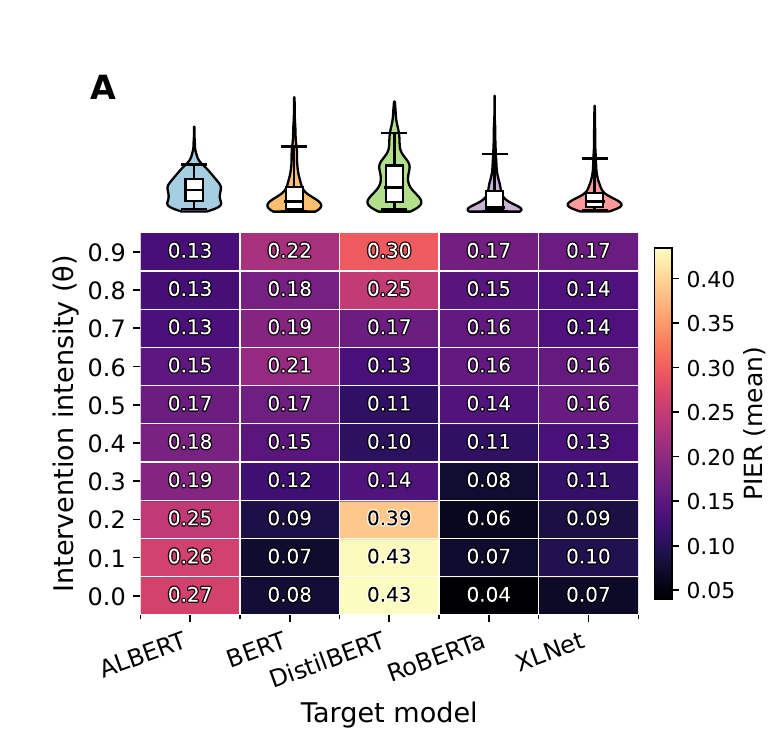}
        % \caption{PIER heatmap across targets and doses}
        \label{fig:fig3a}
    \end{subfigure}
    \hfill
    \begin{subfigure}{0.32\linewidth}
        \centering
        \includegraphics[width=\linewidth]{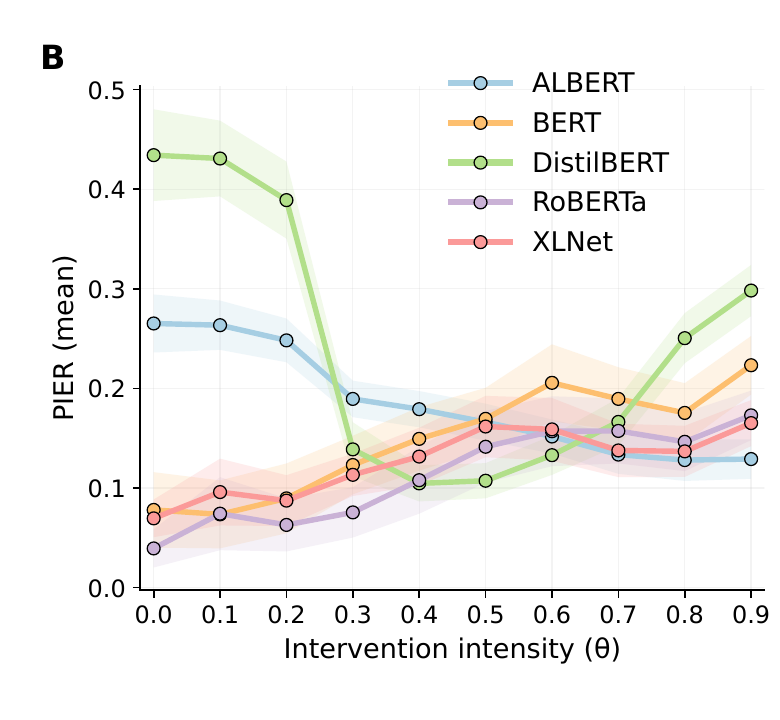}
        % \caption{Dose--response curves of PIER}
        \label{fig:fig3b}
    \end{subfigure}
   \hfill
    \begin{subfigure}{0.32\linewidth}
        \centering
        \includegraphics[width=\linewidth]{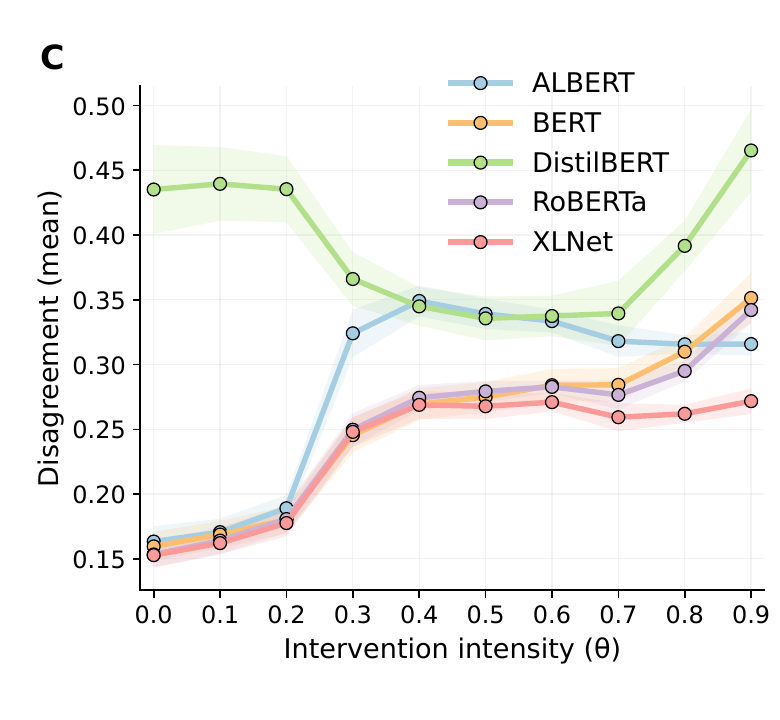}
        % \caption{DISCO routing structure vs.\ dose}
        \label{fig:fig3c}
    \end{subfigure}
   
    \begin{subfigure}{0.32\linewidth}
        \centering
        \includegraphics[width=\linewidth]{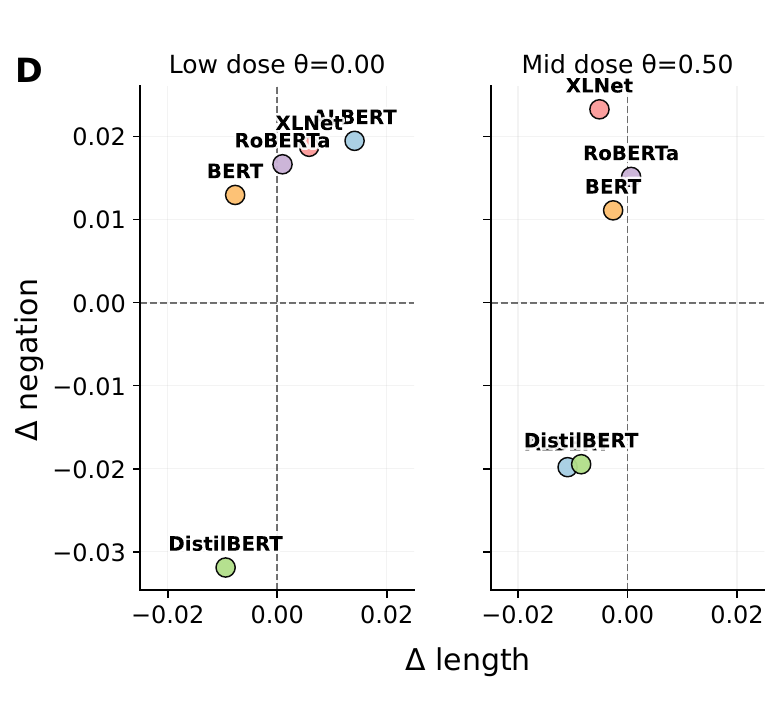}
        % \caption{DISCO routing structure vs.\ dose}
        \label{fig:fig3d}
    \end{subfigure}
 \hfill
    \begin{subfigure}{0.32\linewidth}
        \centering
        \includegraphics[width=\linewidth]{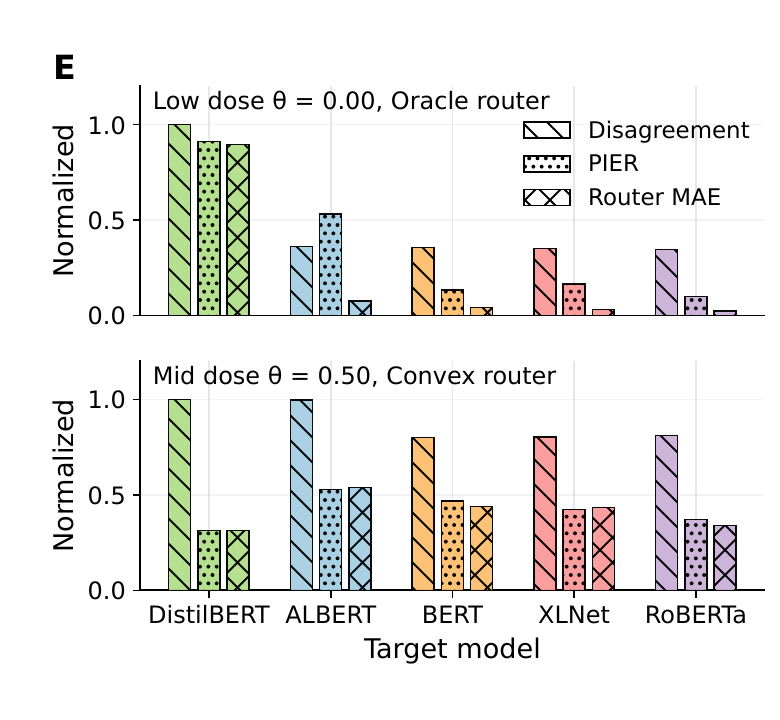}
        % \caption{DISCO routing structure vs.\ dose}
        \label{fig:fig3e}%
    \end{subfigure}
     \hfill
    \begin{subfigure}{0.32\linewidth}
        \centering
        \includegraphics[width=\linewidth]{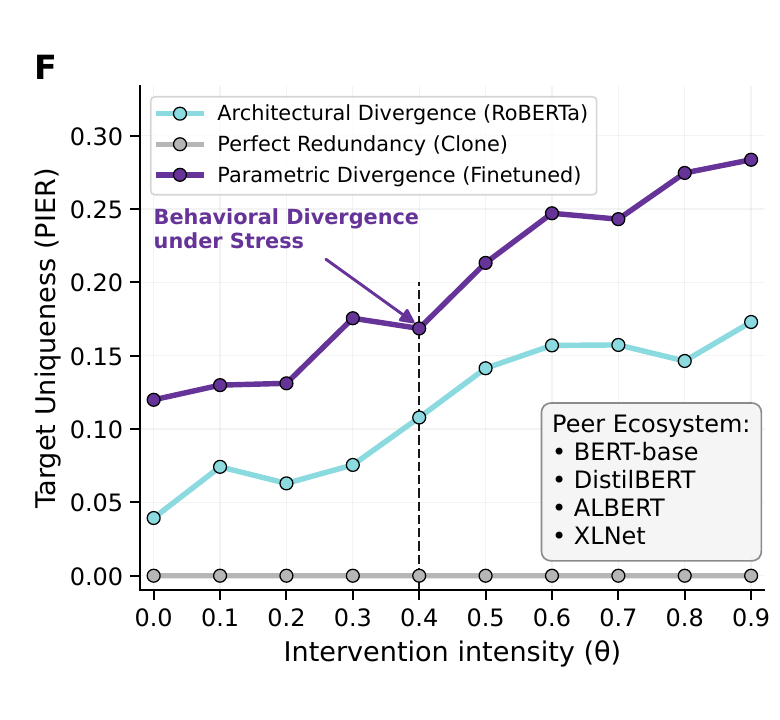}
        % \caption{Clone / variant sanity check}
        \label{fig:fig3f}
    \end{subfigure}

\caption{
\textbf{Auditing an SST-2 ecosystem with DISCO under matched token masking.}
PIER reveals dose-dependent uniqueness (A,B) that can differ from raw disagreement (C).
Context fingerprints at $\theta=0.0$ and $0.5$ localize which linguistic regimes drive residuals (D).
Routing checks validate the interpretation: oracle routing still fails for ALBERT at $\theta=0.0$, while DISCO-weight convex routing matches all targets at $\theta=0.5$ (E).
Controlled targets separate redundancy, architectural divergence, and parametric divergence (F).
}
  \label{fig:llm_audit}
\end{figure}
%TC:endignore

To demonstrate the practical utility of DISCO, we audited a heterogeneous ecosystem of language models on SST-2~\cite{socher2013recursive,wang2019glue} under matched token-masking interventions.
We treated the masking intensity $\theta$ as a controlled Type-B stress coordinate and used each model's positive-class probability as the scalarized response.
We considered a five-model ecosystem comprising \texttt{BERT-base}~\cite{devlin2019bert}, \texttt{DistilBERT}~\cite{sanh2019distilbert}, \texttt{ALBERT}~\cite{lan2020albert}, \texttt{RoBERTa}~\cite{liu2019roberta}, and \texttt{XLNet}~\cite{yang2019xlnet}.
For Fig.~\ref{fig:llm_audit}A--E, we performed a cross-audit on this ecosystem by treating each model in turn as the target and using the remaining four models as its peers, namely an ``establishment'' peer set that changes with the target.
For Fig.~\ref{fig:llm_audit}F, we fixed a single establishment peer set, namely \texttt{BERT-base}, \texttt{DistilBERT}, \texttt{ALBERT}, and \texttt{XLNet}, and we audited three controlled targets to separate sources of uniqueness:
(1) a \texttt{DistilBERT} \emph{Clone}, which is a bit-wise copy of the peer \texttt{DistilBERT} and represents perfect redundancy;
(2) a \texttt{RoBERTa} model, which represents architectural divergence from the BERT lineage; and
(3) a \texttt{DistilBERT} \emph{Finetuned}, which shares the same architecture as \texttt{DistilBERT} but differs by fine-tuning, representing parametric divergence.

Across the cross-audit, Fig.~\ref{fig:llm_audit}A--B shows that uniqueness is strongly dose-dependent.
We observe qualitatively different PIER trajectories across models, and these trajectory shapes carry audit-relevant information that a single scalar score cannot capture.
For example, DistilBERT exhibits a pronounced mid-dose redundancy window where its PIER drops sharply relative to its low-dose behavior, while BERT, RoBERTa, and XLNet show increasing peer-inexpressible residuals as masking becomes more severe.
In contrast, ALBERT exhibits its largest PIER at low dose and then decays as $\theta$ increases.
This heterogeneity supports a key practical message of our framework: we should treat ``uniqueness'' as a property \emph{conditional on a controlled intervention protocol}, not as a static label attached to a model.

A second message in Fig.~\ref{fig:llm_audit}B--C is that ``being different'' is not the same as ``being irreplaceable.''
Disagreement (Fig.~\ref{fig:llm_audit}C) measures the average deviation from individual peers, so it answers a pairwise question.
PIER (Fig.~\ref{fig:llm_audit}B) instead measures the residual after projecting onto the convex hull of peer behaviors, so it answers a substitution question.
These two views can disagree at specific stress levels, and Fig.~\ref{fig:llm_audit}E turns those disagreements into testable, operational claims about routing.

At low dose $\theta=0.0$, Fig.~\ref{fig:llm_audit}B and Fig.~\ref{fig:llm_audit}C give different conclusions for ALBERT.
PIER ranks ALBERT as highly peer-inexpressible, while disagreement does not single it out.
To test what this gap means, we ask the strongest possible substitution question: if we could route each input to the best peer using hindsight, could we match ALBERT's responses?
We implement this as an oracle router that selects, for each example, the peer with the smallest absolute error.
Fig.~\ref{fig:llm_audit}E (top) shows that ALBERT still incurs a large routing error under this oracle policy.
This observation sharpens the interpretation of the low-dose PIER signal.
ALBERT is not merely ``far'' from a particular peer; rather, the ecosystem lacks the behavioral pieces needed to cover ALBERT even with per-example routing.
In other words, the residual we observe is not an artifact of a poor mixture choice, and it reflects a genuine uniqueness margin relative to the available peers.

At mid dose $\theta=0.5$, we observe the converse mismatch.
Disagreement suggests that DistilBERT and ALBERT remain especially different from the rest of the ecosystem, but PIER does not treat them as uniquely peer-inexpressible in this regime.
Here we test the constructive implication of small PIER, namely \emph{actionable redundancy}.
We build a convex router that samples peers using the DISCO projection weights learned from an independent fitting split, and we evaluate its mean absolute error on held-out inputs at $\theta=0.5$.
Fig.~\ref{fig:llm_audit}E (bottom) shows that this simple context-independent routing policy achieves uniformly small error across targets, and DistilBERT is among the easiest models to substitute.
This result explains why disagreement can overstate uniqueness: a model can be moderately far from every single peer while still lying close to their convex hull.
More importantly, it illustrates a practical advantage of PIER over heuristic difference scores.
PIER does not only rank models, but it also produces a concrete replacement policy and it predicts when that policy should work.

Finally, Fig.~\ref{fig:llm_audit}D and Fig.~\ref{fig:llm_audit}F show how our audit goes beyond global ranking.
In Fig.~\ref{fig:llm_audit}D, we summarize how each model's residual shifts across two interpretable linguistic contrasts, namely sentence length and the presence of explicit negation.
At $\theta=0.0$, ALBERT lies in the upper-right quadrant, which indicates that its peer-inexpressible behavior concentrates on longer sentences and negation, while DistilBERT lies in the lower-left quadrant, indicating the opposite trend.
At $\theta=0.5$, the models rearrange, and ALBERT moves toward the lower-left cluster, which shows that stress can change which linguistic regimes drive residual behavior.
This context sensitivity suggests that ecosystem governance does not need to reduce to a binary keep-or-remove decision.
When uniqueness concentrates in specific regimes, we can route those regimes to the unique model and route the remaining traffic to cheaper peers.

In Fig.~\ref{fig:llm_audit}F, we isolate three sources of uniqueness with controlled targets.
The DistilBERT clone stays at near-zero PIER across all $\theta$, which validates that DISCO does not hallucinate uniqueness when perfect redundancy exists.
RoBERTa shows a stable, non-zero trajectory, consistent with architectural divergence.
The fine-tuned DistilBERT variant displays a sharp increase in PIER once $\theta$ enters a higher-stress regime, indicating stress-induced behavioral divergence even within the same architecture.
Taken together, Fig.~\ref{fig:llm_audit} shows how DISCO converts a geometric definition of uniqueness into a practical audit workflow: we can characterize dose-dependent uniqueness, localize it to interpretable contexts, and validate replaceability claims through explicit routing policies.

\subsection*{Use Case Study: Auditing Computer Vision Models}

We evaluate whether PIER provides actionable auditing signals for modern computer vision ecosystems.
We build a peer set spanning standard ImageNet~\cite{deng2009imagenet} classifiers (ResNet~\cite{he2016deep}, EfficientNet~\cite{tan2019efficientnet}, ConvNeXt~\cite{liu2022convnet}, ViT~\cite{dosovitskiy2021image}) and robust~\cite{madry2018towards} or shape-biased~\cite{geirhos2018imagenet} variants, and we probe these models under matched interventions.
We use a scalarized behavioral response (a confidence score for the ground-truth class) and compute PIER within each target model's ecosystem.

% \textbf{Stress can suppress or amplify uniqueness.}
Figure~\ref{fig:cv_audit}A treats adversarial strength $\epsilon$~\cite{goodfellow2015explaining,madry2018towards} as a continuous stress coordinate and reports relative uniqueness PIER$(\epsilon)$/PIER$(0)$.
We see two qualitatively different regimes.
For a standard ResNet-50, uniqueness collapses under attack, with PIER falling to roughly $0.55\times$ its clean baseline.
This pattern suggests that, once inputs cross a robustness threshold, ResNet-50 fails in ways that look increasingly similar to its peers, so the ecosystem offers little incremental behavioral diversity.
In contrast, ConvNeXt shows amplified uniqueness across the same range, with ratios above $1.2\times$, and the robust model shows late-stage amplification.
This contrast matters for auditing.
If we only measure uniqueness on clean data, we can miss models whose distinctive behavior only appears under stress, which is exactly when deployment risk increases.

% \textbf{Context shifts can ``activate'' latent uniqueness.}
In Figure~\ref{fig:cv_audit}B, we switch from adversarial stress to a dataset-context stressor, comparing a texture-natural context to a shape-biased context~\cite{geirhos2018imagenet}.
We plot a paired dumbbell per target model, so the comparison controls for that target's peer set.
Across models, the shape-biased context increases PIER, but the amplification is highly non-uniform.
Shape-biased ResNet variants exhibit the strongest lift, and some standard architectures also show large fold-changes.
This result supports a core claim of PIER: uniqueness is not an intrinsic model attribute, but a \emph{context-relative} property of the model within a given ecosystem.

%TC:ignore
\begin{figure}[t]
    \centering

    \begin{subfigure}{0.32\linewidth}
        \centering
        \includegraphics[width=\linewidth]{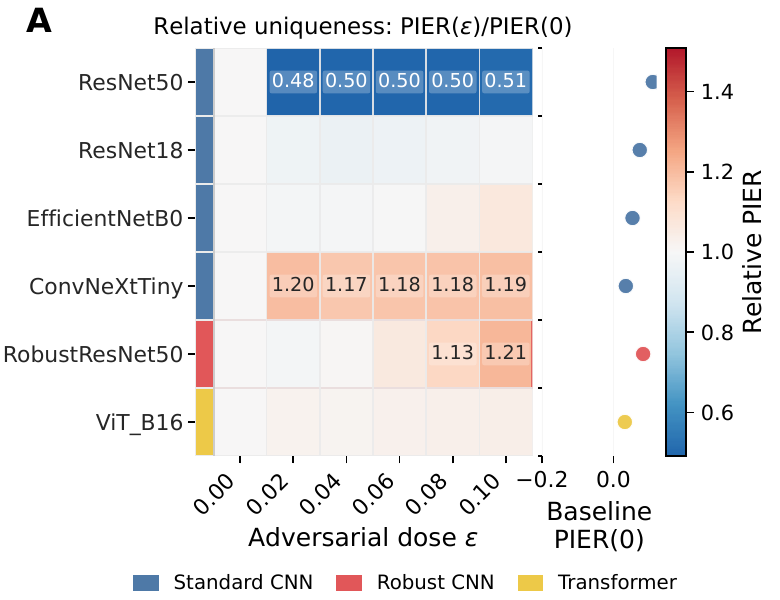}
        \label{fig:fig4a}
    \end{subfigure}
    \hfill
    \begin{subfigure}{0.32\linewidth}
        \centering
        \includegraphics[width=\linewidth]{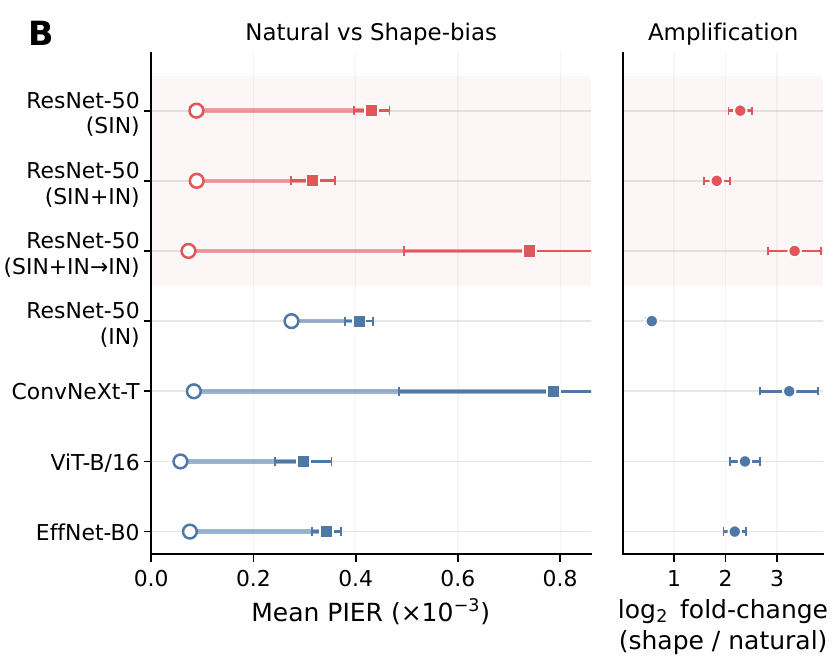}
        \label{fig:fig4b}
    \end{subfigure}
   \hfill
    \begin{subfigure}{0.32\linewidth}
        \centering
        \includegraphics[width=\linewidth]{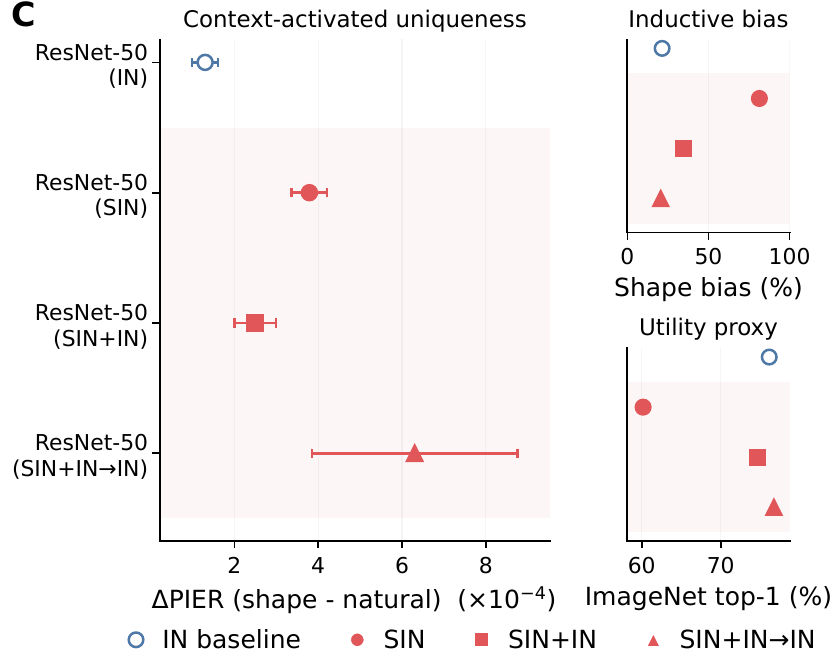}
        \label{fig:fig4c}
    \end{subfigure}

    \begin{subfigure}{0.49\linewidth}
        \centering
        \includegraphics[width=\linewidth]{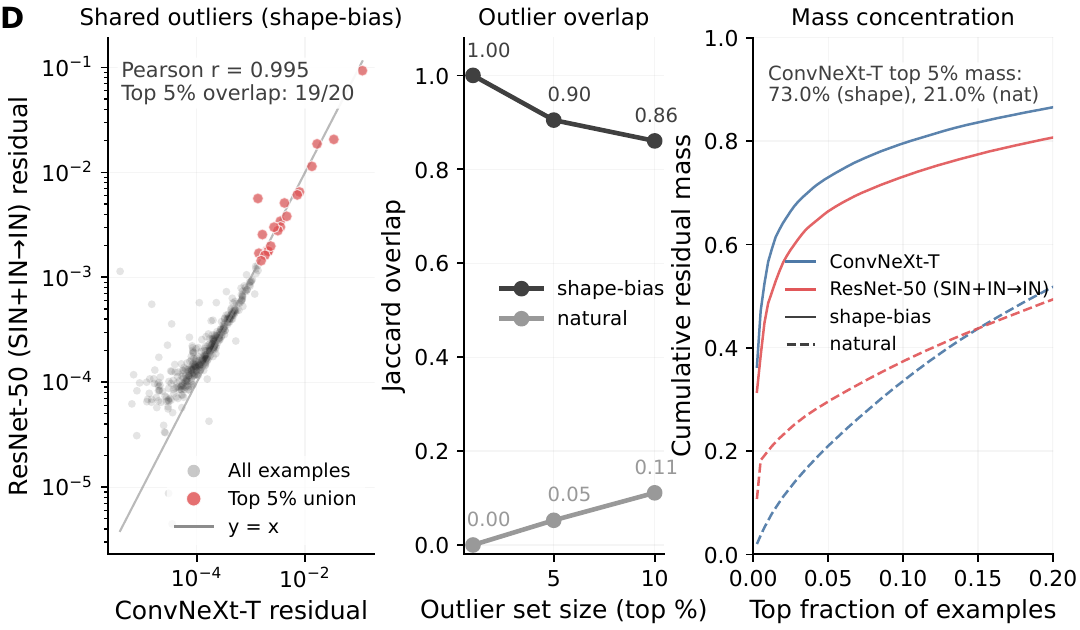}
        \label{fig:fig4d}
    \end{subfigure}
     \hfill
    \begin{subfigure}{0.49\linewidth}
        \centering
        \includegraphics[width=\linewidth]{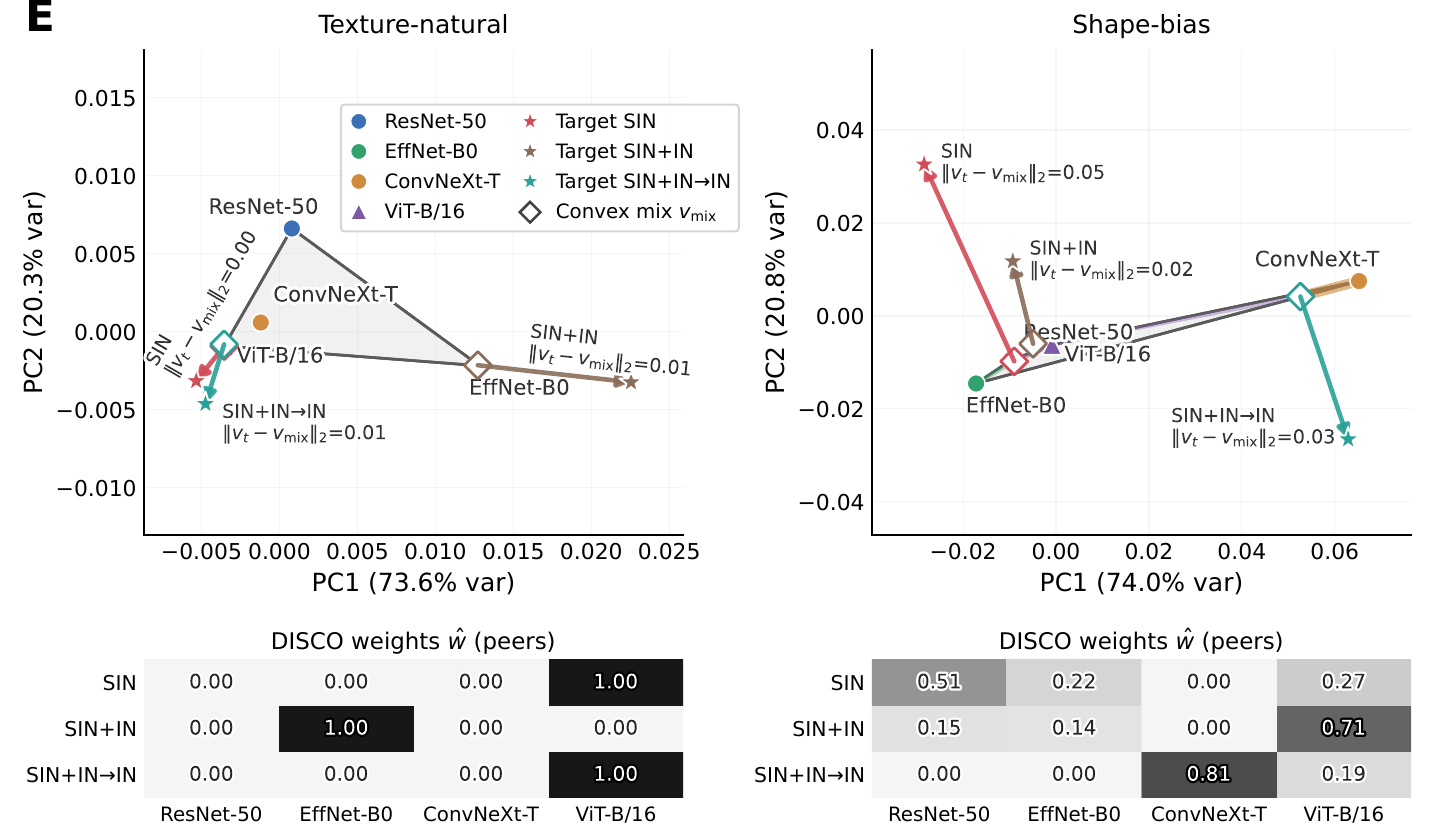}
        \label{fig:fig4e}
    \end{subfigure}

\caption{
\textbf{Auditing vision ecosystems under controlled stressors.}
(A) Relative uniqueness across a matched adversarial dose $\epsilon$, reported as PIER$(\epsilon)$/PIER$(0)$, with baseline PIER$(0)$ shown at right.
(B) Context audit comparing texture-natural versus shape-biased evaluation, plus the within-model amplification $\log_2(\mathrm{PIER}_{\mathrm{shape}}/\mathrm{PIER}_{\mathrm{nat}})$.
(C) Context-activated uniqueness $\Delta$PIER decouples from both inductive bias (shape bias score) and a utility proxy (ImageNet top-1).
(D) Heavy-tail diagnostics show that a small set of shared outliers dominates residual mass under shape-bias.
(E) Convex-hull geometry: targets move outside the peer hull in the shape-bias context, and DISCO weights explain the closest convex surrogate.
}
  \label{fig:cv_audit}
\end{figure}
%TC:endignore

% \textbf{$\Delta$PIER is not a proxy for bias or utility.}
Figure~\ref{fig:cv_audit}C summarizes context-activated uniqueness as $\Delta$PIER $=$ PIER$_{\mathrm{shape}}-$PIER$_{\mathrm{nat}}$ and compares it against two common summary axes.
First, it does not track an inductive-bias score (shape bias), even though the experimental manipulation directly targets shape versus texture cues.
Second, it does not track ImageNet top-1 accuracy, which we use as a simple utility proxy.
This decoupling matters for governance.
A team that optimizes for accuracy or for a single bias metric can still assemble an ecosystem with large redundancy, or overlook models whose behavior becomes uniquely risky or uniquely valuable under shift.
PIER instead measures substitutability against a peer hull, so it captures complementarity that these one-dimensional summaries miss.

% \textbf{Uniqueness concentrates on a small set of high-leverage examples.}
To explain why certain models show large context amplification, Figure~\ref{fig:cv_audit}D inspects per-example residuals between each target and its best convex surrogate.
Under the shape-bias context, residuals become heavy-tailed: the top few percent of examples account for most of the residual mass, and ConvNeXt and the shape-biased ResNet share these outliers almost perfectly.
Under the texture-natural context, the same overlap collapses toward zero.
This diagnostic provides a concrete interpretation for ``context-activated'' uniqueness.
The context shift does not uniformly change behavior across the dataset.
Instead, it activates a rare but decisive set of test cases where the ecosystem's failure modes align or diverge, and these cases dominate the uniqueness signal.

% \textbf{Convex-hull geometry makes auditing interpretable.}

Finally, Fig.~\ref{fig:cv_audit}E visualizes the geometric object that PIER measures.
We embed the peer models and each target model into a 2D PCA space for visualization, shade the
peers' convex hull, and plot the DISCO-induced convex-combination point that best approximates the
target in this 2D projection. In the texture-natural context, target points typically lie near (or
inside) the peers' hull, and the convex-combination point nearly overlaps with the target,
indicating near-complete substitutability. In the shape-bias context, targets more often fall
outside the peers' hull; the gap becomes larger and the weights redistribute, revealing which peers
explain the closest surrogate and what residual behavior remains irreducible. This panel closes the
loop between theory and practice: PIER behaves as a geometric distance-to-hull signal, and DISCO
produces an explicit, auditable substitution certificate.

Together, Figure~\ref{fig:cv_audit} shows that PIER supports high-level governance questions that accuracy and bias metrics do not answer.
We can stress-test ecosystems, identify contexts that activate uniqueness, localize the small set of examples that drive irreducible behavior, and extract interpretable substitution structures that explain which models the ecosystem can safely consolidate.

\FloatBarrier

\subsection*{Use Case Study: Multi-City Traffic Predictions}

To stress-test DISCO in a high-stakes \emph{model-governance} setting, we audited an ecosystem of $N=31$ city-level traffic forecasters.
Each city model served as a target in turn, and we treated the remaining city models together with a pooled \texttt{GLOBAL} model as its peers.
For every target city, DISCO returned (i) a uniqueness score, PIER, and (ii) an interpretable convex decomposition over peers, $\hat{w}$, which we deployed as a \emph{convex router} to substitute the target.
This setting let us connect geometric uniqueness to concrete operational questions, namely: which city models are redundant, which ones are irreplaceable, and how much consolidation we can achieve under a fixed performance budget.

%TC:ignore
\begin{figure}[t]
    \centering

    \begin{subfigure}{0.37\linewidth}
        \centering
        \includegraphics[width=\linewidth]{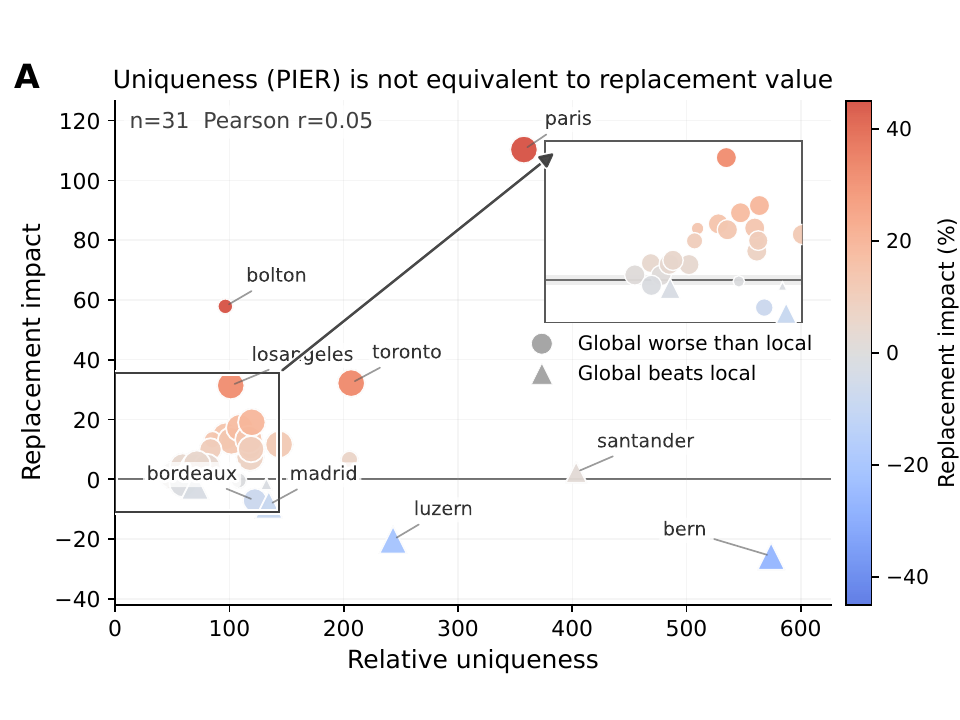}
        \label{fig:fig5a}
    \end{subfigure}
    \hfill
    \begin{subfigure}{0.62\linewidth}
        \centering
        \includegraphics[width=\linewidth]{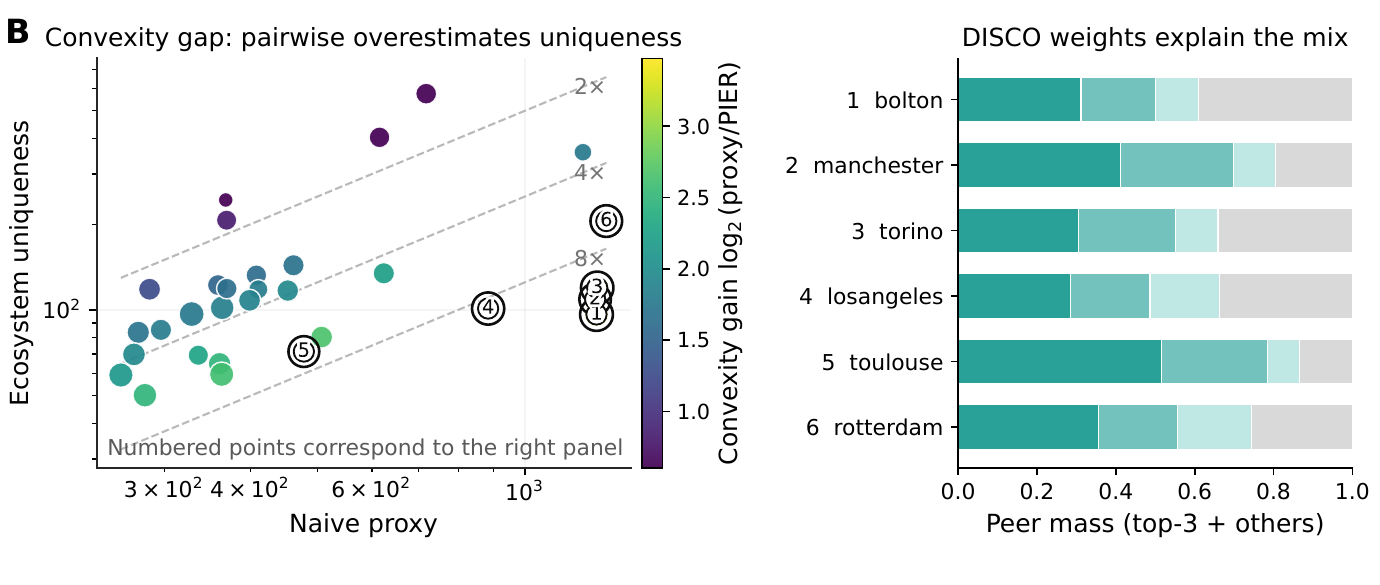}
        \label{fig:fig5b}
    \end{subfigure}

    \begin{subfigure}{0.33\linewidth}
        \centering
        \includegraphics[width=\linewidth]{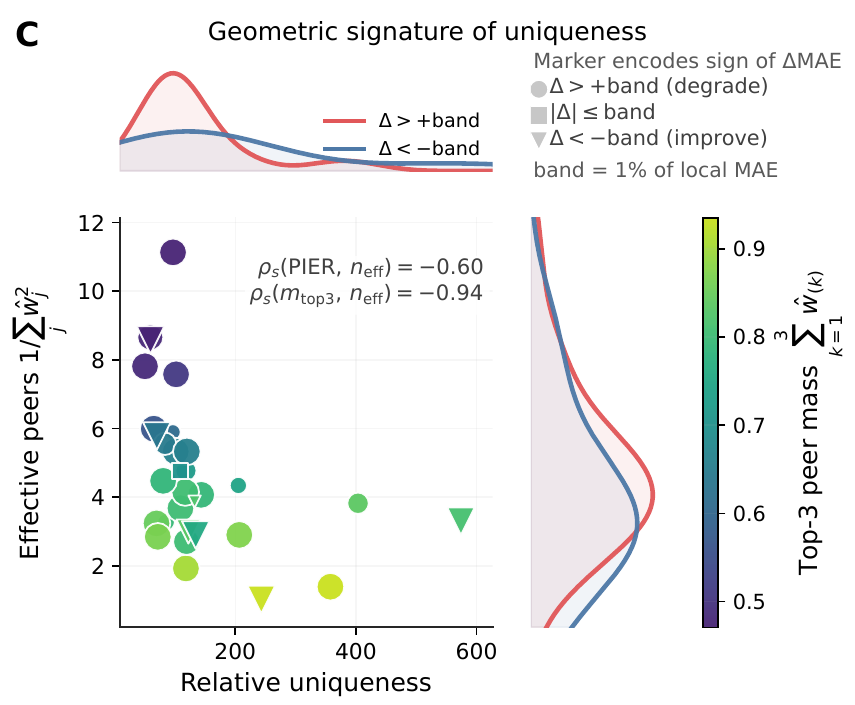}
        \label{fig:fig5c}
    \end{subfigure}
    \hfill
    \begin{subfigure}{0.33\linewidth}
        \centering
        \includegraphics[width=\linewidth]{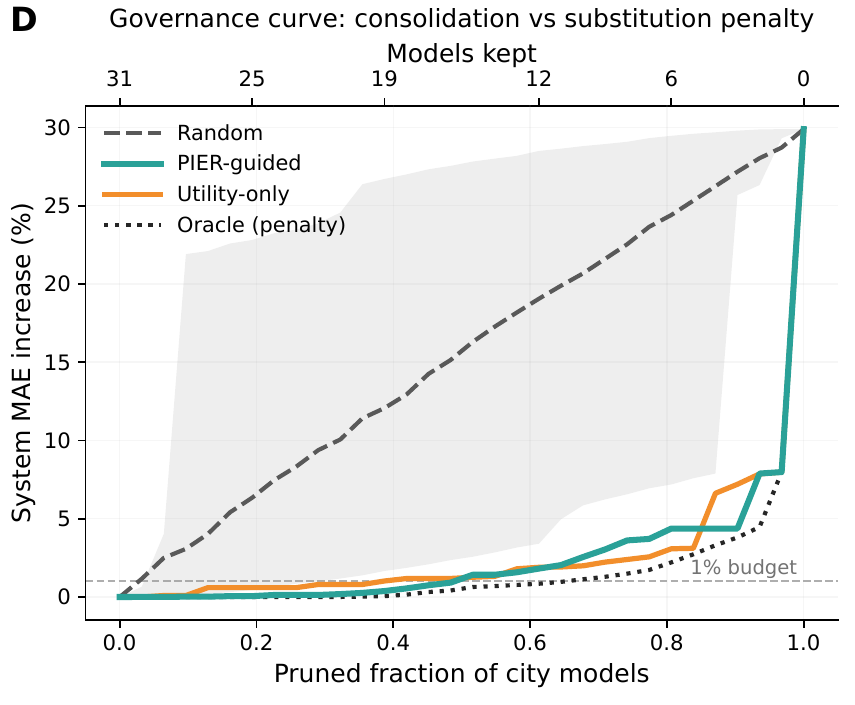}
        \label{fig:fig5d}
    \end{subfigure}
    \hfill
    \begin{subfigure}{0.33\linewidth}
        \centering
        \includegraphics[width=\linewidth]{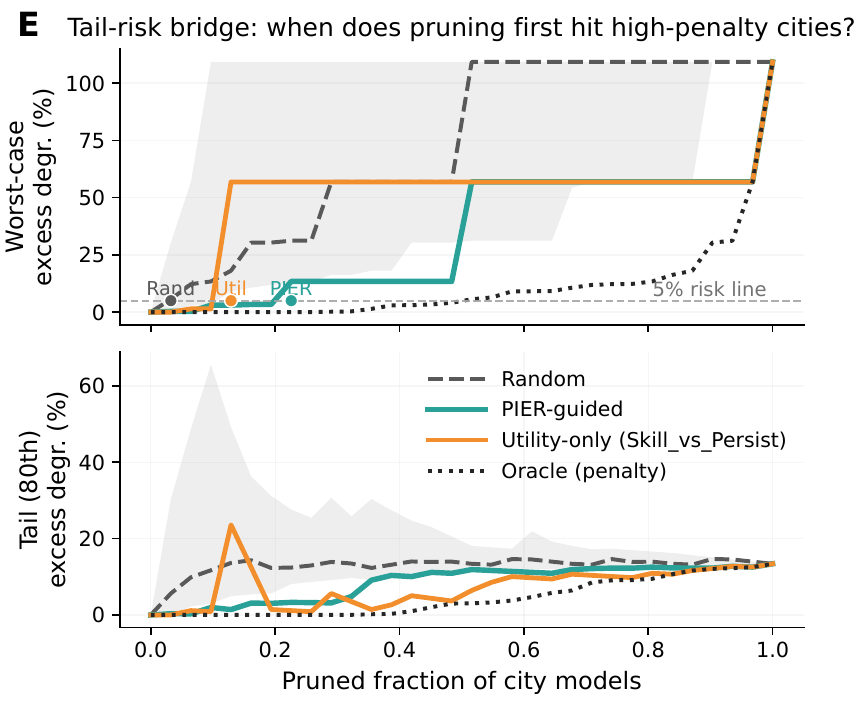}
        \label{fig:fig5e}
    \end{subfigure}

\caption{
\textbf{Auditing traffic models trained on multiple cities.}
(A) Relative uniqueness versus signed replacement impact when substituting a local city model with the DISCO convex router (inset shows the dense region; marker shape indicates whether \texttt{GLOBAL} beats the local model).
(B) Convexity gap between a pairwise proxy and ecosystem uniqueness (left), with DISCO weights that explain representative high-gap cases (right).
(C) Geometric signature linking uniqueness to convex-mix complexity; we report Spearman's rank correlation $\rho_s$.
(D) Consolidation curve: system MAE increase as we prune city models under different pruning orders.
\emph{Utility-only} ranks cities by their standalone forecasting skill relative to
a persistence baseline, ignoring substitution/replacement penalties. \emph{Oracle (penalty)} ranks
cities using access to the true replacement penalty (available only for evaluation, not in practice).
(E) Tail-risk bridge: worst-case and tail excess degradation along the pruning path.}
  \label{fig:multicity_audit}
\end{figure}
%TC:endignore

% \textbf{Uniqueness does not equal replacement value.}
In Fig.~\ref{fig:multicity_audit}A we plot \emph{relative uniqueness} on the $x$-axis (PIER normalized by the typical flow scale) against the \emph{replacement impact} on the $y$-axis, which we define as the signed MAE change when we replace a target city model with its convex router.
A positive impact means substitution degrades accuracy, while a negative impact means the router improves over the local model.
We observe essentially no linear association (Pearson $r=0.05$), which highlights a central message of our paper: PIER measures \emph{behavioral} uniqueness relative to peers, not task utility.
This decoupling exposes three practically distinct regimes that a single utility score cannot separate.
First, many cities sit near the origin in the inset, which indicates low uniqueness and small replacement impact; these models are natural consolidation candidates.
Second, a small set of cities occupy the upper-right region (for example, \texttt{paris}), where high uniqueness co-occurs with large positive replacement impact; these models are irreplaceable because no peer mixture reproduces their behavior without paying a clear error penalty.
Third, we find \emph{unique-but-unhealthy} cities in the lower-right region (for example, \texttt{bern}), where PIER is high but replacement impact is negative and the \texttt{GLOBAL} model already outperforms the local model (triangle markers).
In this regime, uniqueness reflects idiosyncratic failure modes rather than valuable specialization, and DISCO flags these cases precisely because it separates ecosystem geometry from utility.

% \textbf{Convexity gap: why pairwise disagreement overstates uniqueness.}
Fig.~\ref{fig:multicity_audit}B demonstrates the mechanism behind many apparent ``unique'' cities under naive metrics.
The left panel compares an ecosystem-level uniqueness score (PIER) to a pairwise proxy that ignores mixtures, and we color points by the \emph{convexity gain} $\log_2(\text{proxy}/\text{PIER})$.
Large gains indicate cities that look far from every single peer, yet fall close to the \emph{convex hull} of the peer set.
These are exactly the cases where pairwise comparisons systematically overestimate uniqueness, because the true substitute is not one peer but a sparse mixture of several.
The right panel makes this geometric fact operational: DISCO returns $\hat{w}$ and shows that high-gap cities admit a compact top-$3$ decomposition that explains most of the mixture mass.
This weight-level explanation matters for governance, because it converts ``this city is redundant'' into a concrete substitution plan and reveals which peer models form the effective backbone of the ecosystem.

% \textbf{Geometric signature: unique cities lie on sparse faces of the peer hull.}
Fig.~\ref{fig:multicity_audit}C ties PIER back to convex geometry using two complementary notions of mixture complexity.
We summarize how many peers the router effectively uses via
$n_{\mathrm{eff}} := 1/\sum_j \hat{w}_j^2$, and we summarize weight concentration via the top-$3$ mass $m_{\mathrm{top3}} := \sum_{k=1}^3 \hat{w}_{(k)}$ (sorted in descending order).
We observe a clear monotone relationship: as relative uniqueness grows, the convex mix becomes less diffuse (Spearman $\rho_s(\text{PIER}, n_{\mathrm{eff}})=-0.60$), and $m_{\mathrm{top3}}$ almost deterministically predicts $n_{\mathrm{eff}}$ ($\rho_s(m_{\mathrm{top3}}, n_{\mathrm{eff}})=-0.94$).
This pattern matches the convex-hull view of DISCO.
When a target lies outside the peer hull, its projection typically lands on a low-dimensional face supported by a small subset of peers, which concentrates $\hat{w}$ and reduces $n_{\mathrm{eff}}$.
Importantly, the marker encoding reveals that high uniqueness can correspond to either degradation or improvement under substitution, which reinforces that PIER is not a utility proxy.
Instead, PIER and the weight geometry tell us \emph{how} a model differs from the ecosystem and \emph{how} substitution will depend on a small set of supporting peers.

% \textbf{PIER enables budgeted consolidation curves, not just rankings.}
Fig.~\ref{fig:multicity_audit}D turns the audit into a decision procedure.
We progressively prune city models and substitute each removed model with a convex router built from the remaining ecosystem.
We then measure the system-level MAE increase as a function of the pruned fraction.
A random pruning order degrades steadily, while the PIER-guided order stays near the oracle lower bound for most of the path and exhibits a sharp knee only when pruning starts to remove high-PIER cities.
A utility-only heuristic based on forecasting skill also looks competitive on mean degradation, but it provides a fundamentally different guarantee.
Utility is city-intrinsic, while substitution depends on ecosystem coverage.
PIER directly targets that coverage by ordering models according to convex-hull novelty, which is why it produces a clean governance curve that supports statements of the form: ``we can remove a large fraction of models while staying within a fixed MAE budget.''

% \textbf{Tail-risk is where ecosystem geometry matters most.}
Finally, Fig.~\ref{fig:multicity_audit}E shows why we need a geometric notion of uniqueness even when a utility heuristic seems close on average.
We report the \emph{excess degradation} beyond a $1\%$ local-MAE band and track both the worst-case city and the 80th-percentile tail as pruning proceeds.
The worst-case curve exhibits step-like jumps that reveal a heavy-tailed substitution-risk landscape.
Utility-only pruning hits a high-penalty city early, producing an immediate large worst-case jump, while PIER-guided pruning delays this first catastrophic event to a much later pruning fraction.
This behavior is consistent with the convex-hull interpretation: high-PIER cities sit on extreme faces that the ecosystem cannot reconstruct, so pruning them early exposes the system to abrupt SLA violations.
By contrast, PIER-guided consolidation preferentially removes interior, well-covered cities first, which makes the tail-risk profile substantially more stable.

Taken together, Fig.~\ref{fig:multicity_audit} demonstrates how DISCO turns a heterogeneous set of deployed predictors into a coherent geometric object.
We use that geometry to (i) decouple uniqueness from utility, (ii) avoid false uniqueness from pairwise proxies, (iii) produce interpretable substitution backbones via $\hat{w}$, and (iv) deliver governance curves that quantify consolidation budgets and tail-risk tradeoffs in a directly actionable way.

\section*{Discussion}

We introduced a statistical framework for quantifying model uniqueness in heterogeneous AI ecosystems via
in-silico quasi-experimental design. By separating intrinsic model identity (Type-A treatment) from extrinsic
input interventions (Type-B treatment), and by applying matched Type-B patterns across all models, we obtain
a convex peer expressivity class and a residual notion of peer-inexpressible behavior. This yields a population
uniqueness functional (PIER) with a direct operational meaning: whether a target can be substituted by a
probabilistic routing policy over its peers. Under mild regularity conditions, PIER can be estimated reliably by
the DISCO estimator.

On the statistical side, the passive theory establishes that, under fixed designs, DISCO behaves like a standard
projection-based estimator: the projection weights and uniqueness converge to their population counterparts and
admit asymptotic normal approximations. Finite-sample bounds further clarify how design choices (e.g., anchor
radius and dose-grid spacing) control estimation error, supporting the use of PIER as a calibrated audit
statistic rather than an arbitrary score.

A second advantage of in-silico design is algorithmic: when the auditor controls interventions, auditing can be
sample-efficient. In a local linear structural model, a small set of carefully chosen design points suffices for
exact uniqueness detection in the noiseless regime. With noise, repeated queries at these points yield explicit
non-asymptotic sample-complexity guarantees whose dependence on the noise level, uniqueness margin, and
confidence is minimax optimal. These results characterize when uniqueness auditing is both principled and
efficient.

Several structural properties make the convex formulation conservative yet actionable for ecosystem governance.
Monotonicity in the peer set connects PIER to pruning and consolidation decisions, while the routing
interpretation of projection weights provides an explicit, auditable substitution certificate. At the same time,
our non-identifiability results and the separation between PIER and Shapley-based attribution emphasize that
uniqueness auditing is distinct from credit assignment and that matched intervention control is essential.

Empirically, simulations further suggest an ``ecosystem saturation'' effect, where
average uniqueness decreases rapidly as the peer set grows and then exhibits diminishing returns. We conjecture
that analogous phase-transition behavior persists under suitable sampling assumptions on low-dimensional task
manifolds; making this precise and deriving dimension-dependent rates is an interesting direction for future work.

From a practical standpoint, the framework suggests a two-stage audit pipeline in large model ecosystems: a
coarse screening stage to identify low-uniqueness candidates, followed by targeted local audits and robustness
analyses of the PIER trajectory for models of interest. Future work includes incorporating richer structural
priors in the expressivity class (while retaining conservative aggregation), developing adaptive context
selection strategies that focus audit budget on consequential regions of the intervention space, and connecting
uniqueness capacity to downstream trustworthy-AI notions such as robustness, privacy, and policy constraints.
Applying the framework to industrial-scale ecosystems of large language models~\cite{bommasani2021opportunities},
vision models, and multimodal systems will require careful engineering of scalarization maps, intervention
families, and approximate inference under resource constraints.

%===========================
% PART 2: METHODS SECTION
% (to be appended at the end of the main-text PDF)
%===========================

\section*{Methods}

\subsection*{Model Ecosystems and Scalarised Responses}

We consider an input space $\Xcal$ and a dose space $\ThetaSpace$. A deterministic intervention map $T : \ThetaSpace \times \Xcal \to \Xcal$ maps an input $x \in \Xcal$ and a dose parameter $\theta \in \ThetaSpace$ to a perturbed input $T(\theta,x)$. A model ecosystem consists of functions $f_j : \Xcal \to \Ycal$ indexed by $j \in \mathcal{J} = \{1,\dots,N\}$. A scalarisation map $g : \Ycal \to \R$ maps model outputs to real-valued scores.

For each model $j$ and input–dose pair $(x,\theta)$, the scalarised response is defined by
\[
Y_j(x,\theta) = g\big(f_j(T(\theta,x))\big) \in \R.
\]
A target index $t \in \mathcal{J}$ is fixed and its peer set is $\Peers = \mathcal{J} \setminus \{t\}$. The vector of peer responses at $(x,\theta)$ is
\[
\Phib_{-t}(x,\theta) = (Y_j(x,\theta))_{j \in \Peers} \in \R^{|\Peers|}.
\]

\subsection*{In-Silico Quasi-Experimental Design}

A context variable $\Ctx$ takes values in a context space $\Ccalctx$. The global design distribution $\Pdesign$ is a probability measure on $\Xcal \times \ThetaSpace \times \Ccalctx$ describing how inputs, doses and contexts are generated. Conditioning on $\Ctx = c$ yields a distribution on $\Xcal \times \ThetaSpace$. Throughout the Methods we fix such a context $c$ and suppress it from the notation.

Within the fixed context, two design distributions play distinct roles. A fitting design $\Pfit$ is a probability measure on $\Xcal \times \ThetaSpace$ used to generate fitting pairs $(X^{\mathrm{fit}},\vartheta^{\mathrm{fit}})$ at which model responses are observed for learning the peer surrogate. An evaluation design $\Peval$ is a probability measure on $\Xcal \times \ThetaSpace$ used to generate evaluation pairs $(X^{\mathrm{eval}},\vartheta^{\mathrm{eval}})$ at which PIER and uniqueness are assessed.

The honesty assumption states that, conditional on the context, the fitting sample $(X_i^{\mathrm{fit}},\vartheta_i^{\mathrm{fit}})_{i=1}^m$ is drawn independently from $\Pfit$, the evaluation sample $(X_i^{\mathrm{eval}},\vartheta_i^{\mathrm{eval}})_{i=1}^n$ is drawn independently from $\Peval$, and the two samples are independent. This ensures that the same data are not used both to fit the convex peer surrogate and to evaluate PIER.

\subsection*{Peer Expressivity, Population PIER and Uniqueness}

For the fixed context and fitting design $\Pfit$, consider the Hilbert space $L^2(\Pfit)$ of measurable functions $h : \Xcal \times \ThetaSpace \to \R$ with finite second moment under $\Pfit$. The inner product is
\[
\langle h_1,h_2 \rangle = \E_{\Pfit}\big[h_1(X,\vartheta) h_2(X,\vartheta)\big],
\]
with induced norm $\|h\|_2 = \sqrt{\langle h,h \rangle}$.

For a weight vector $\wb \in \R^{|\Peers|}$, define the peer combination function
\[
h_{\wb}(x,\theta) = \wb^\top \Phib_{-t}(x,\theta).
\]
The convex peer expressivity class is
\[
\Ccal = \{h_{\wb} : \wb \in \Simplex\} \subset L^2(\Pfit),
\]
where $\Simplex$ is the simplex of non-negative vectors in $\R^{|\Peers|}$ that sum to one. This class collects all behaviours reachable by stochastic routing over peers with a context-independent routing distribution.

The population residual of a weight vector $\wb$ is the mean-squared deviation between the target and the convex peer combination,
\[
L(\wb) = \E_{\Pfit}\big[(Y_t(X,\vartheta) - \wb^\top \Phib_{-t}(X,\vartheta))^2\big].
\]
Under an identifiability condition, the residual function $L$ has a unique minimiser $\wb^*$ in the relative interior of $\Simplex$, and the residual is locally strongly convex around $\wb^*$. The function $h_{\wb^*}$ is the $L^2(\Pfit)$-projection of $Y_t$ onto the closure of $\Ccal$.

The population peer-inexpressible residual (PIER) is defined pointwise by
\[
\PIER(x,\theta) = Y_t(x,\theta) - \wb^{*\top} \Phib_{-t}(x,\theta).
\]
By construction, $\PIER \in L^2(\Pfit)$ and is orthogonal to the closure of $\Ccal$ in $L^2(\Pfit)$.

To aggregate PIER into a scalar measure, an evaluation design $\Peval$ on $\Xcal \times \ThetaSpace$ is introduced. Under an integrability condition ensuring that $\E_{\Peval}[|\PIER(X,\vartheta)|] < \infty$, the population uniqueness functional is defined as
\[
\Uniqueness = \E_{\Peval}\big[|\PIER(X,\vartheta)|\big],
\]
where $(X,\vartheta) \sim \Peval$. The functional $\Uniqueness$ depends on the target, the peer set, the fitting design via $\wb^*$, and the evaluation design.

\subsection*{The DISCO Estimator}

The DISCO estimator implements empirical counterparts of $\wb^*$, $\PIER$ and $\Uniqueness$ using finite fitting and evaluation samples. For the fixed context, a fitting sample $(X_i^{\mathrm{fit}},\vartheta_i^{\mathrm{fit}})_{i=1}^m$ is drawn independently from $\Pfit$. For each model $j$, the scalarised responses
\[
Y_{j,i}^{\mathrm{fit}} = Y_j(X_i^{\mathrm{fit}},\vartheta_i^{\mathrm{fit}})
\]
are computed. The peer responses are assembled into vectors
\[
\Phib_{-t,i}^{\mathrm{fit}} = \Phib_{-t}(X_i^{\mathrm{fit}},\vartheta_i^{\mathrm{fit}}),
\]
which are stacked row-wise into a design matrix $\Xb_{-t}^{\mathrm{fit}} \in \R^{m \times |\Peers|}$. The target responses form a vector $y_t^{\mathrm{fit}} \in \R^m$.

For a regularisation parameter $\lambda_m \ge 0$, the empirical projection weights are defined as any measurable solution to the convex optimisation problem
\[
\widehat{\wb} \in \argmin_{\wb \in \Simplex}
\left\{
\frac{1}{m} \sum_{i=1}^m \big(Y_{t,i}^{\mathrm{fit}} - \wb^\top \Phib_{-t,i}^{\mathrm{fit}}\big)^2
+ \lambda_m \|\wb\|_2^2
\right\}.
\]
The regularisation term stabilises the estimator in finite samples and is chosen so that $\lambda_m \to 0$ and $m \lambda_m \to 0$ as $m$ increases for the asymptotic theory.

An independent evaluation sample $(X_i^{\mathrm{eval}},\vartheta_i^{\mathrm{eval}})_{i=1}^n$ is drawn from $\Peval$. For each model $j$,
\[
Y_{j,i}^{\mathrm{eval}} = Y_j(X_i^{\mathrm{eval}},\vartheta_i^{\mathrm{eval}}),
\qquad
\Phib_{-t,i}^{\mathrm{eval}} = \Phib_{-t}(X_i^{\mathrm{eval}},\vartheta_i^{\mathrm{eval}})
\]
are computed. The empirical PIER at the $i$-th evaluation point is
\[
\hPIER{}_i = Y_{t,i}^{\mathrm{eval}} - \widehat{\wb}^\top \Phib_{-t,i}^{\mathrm{eval}}.
\]
The empirical uniqueness functional is the average absolute PIER,
\[
\hUniqueness = \frac{1}{n} \sum_{i=1}^n |\hPIER{}_i|.
\]

In practice, evaluation points are often associated with nearby anchor points in the input–dose space, and responses at anchors are used as proxies for responses at evaluation points. Under Lipschitz regularity of the scalarised responses, the error incurred by substituting anchors for exact evaluation points can be bounded in terms of the distance between each evaluation point and its anchor and the dose-grid spacing.

\subsection*{Local Linear Structural Model and Active Auditing Protocol}

For active auditing analyses, a local linear structural model is considered within a fixed context. A feature map $\phi : \Xcal \times \ThetaSpace \to \R^d$ of dimension $d \ge 1$ is assumed known, and each scalarised response admits a linear parametrisation,
\[
Y_j(x,\theta) = \phi(x,\theta)^\top \beta_j,
\]
for some coefficient vector $\beta_j \in \R^d$. The set of peer coefficients is
\[
\mathcal{B}_{\mathrm{peer}} = \{\beta_j : j \in \Peers\} \subset \R^d,
\]
and their convex hull is denoted $\mathcal{H}_{\mathrm{peer}}$. The target coefficient vector is $\beta_t \in \R^d$.

The uniqueness margin is defined as the Euclidean distance from the target coefficient to the convex hull of peer coefficients,
\[
\gamma = \inf_{\wb \in \Simplex} \left\|\beta_t - \sum_{j \in \Peers} w_j \beta_j \right\|_2.
\]
The target is non-unique in this model if and only if $\gamma = 0$.

To perform active auditing, the auditor chooses a set of $d$ design points $(x_\ell,\theta_\ell)$ such that the resulting feature matrix
\[
\Phi = \begin{bmatrix}
\phi(x_1,\theta_1)^\top\\
\vdots\\
\phi(x_d,\theta_d)^\top
\end{bmatrix}
\in \R^{d \times d}
\]
is invertible. In a noiseless regime, querying each model once at these design points yields response vectors $y_j$ satisfying $y_j = \Phi \beta_j$, so that $\beta_j = \Phi^{-1}y_j$ for all $j$. Uniqueness reduces to a convex feasibility problem in coefficient space.

In a noisy regime, each model evaluation at a design point $(x_\ell,\theta_\ell)$ yields a noisy observation
\[
\widetilde{Y}_j(x_\ell,\theta_\ell) = Y_j(x_\ell,\theta_\ell) + \varepsilon_j(x_\ell,\theta_\ell),
\]
where the noise variables are independent across models, design points and repetitions, have zero mean and are sub-Gaussian with variance proxy $\sigma^2$. The auditor queries each model $r$ times at each design point and averages the responses to form
\[
\bar{y}_{j,\ell} = \frac{1}{r} \sum_{s=1}^r \widetilde{Y}_j^{(s)}(x_\ell,\theta_\ell).
\]
The averaged response vector $\bar{y}_j$ is then used to estimate the coefficients by $\widehat{\beta}_j = \Phi^{-1} \bar{y}_j$. The empirical convex hull of peer coefficients is formed as $\operatorname{conv}\{\widehat{\beta}_j : j \in \Peers\}$, and the empirical distance from $\widehat{\beta}_t$ to this hull is compared to a threshold based on a presumed margin $\gamma$.

%TC:ignore

\subsection*{Experimental Protocols}

Synthetic experiments can be designed to validate the theoretical predictions. In one class of experiments, an input space is sampled from a low-dimensional Euclidean distribution and synthetic scalarised responses are generated according to prescribed coefficient vectors and feature maps, with additive noise. DISCO is applied with various fitting and evaluation sample sizes, anchor radii and dose-grid resolutions, and empirical behaviour of $\widehat{\wb}$, $\hUniqueness$ and anchor-based PIER estimates is compared to Monte Carlo approximations of population quantities. In another class of experiments, the local linear structural model is enforced exactly by choosing $\phi$ and $\beta_j$ directly, and the active auditing protocol is used to recover coefficients and test uniqueness under varying noise levels and repetition counts. These setups allow empirical examination of convergence, normal approximation, design-error scaling and the sample-complexity bounds derived in the Supplementary Information.

\bibliography{sn-bibliography}% common bib file

%TC:endignore

\newpage
\begin{appendices}

\par\vspace{0.2cm}
\begin{center}
    {\Large \bf Appendix: Quantifying Model Uniqueness in Heterogeneous AI Ecosystems}
\end{center}
\par\vspace{0.2cm}

\par\vspace{1em}
{\small
\setcounter{tocdepth}{2}
\startcontents[sections]
\printcontents[sections]{l}{1}{}
}

\section{Passive Theory for DISCO}
\label{app:passive_theory}

This note provides detailed proofs of the passive statistical guarantees and structural properties of the DISCO estimator described in the main text. Throughout, the context value is fixed and suppressed from the notation. Expectations and probabilities are conditional on this context.

\subsection{Consistency of DISCO Under Fixed Designs}
\label{app:consistency}

We first establish that, under mild assumptions, the empirical projection weights $\widehat{\wb}$ converge in probability to the population projection weights $\wb^*$ and the empirical uniqueness functional $\hUniqueness$ converges in probability to the population uniqueness $\Uniqueness$ as the sizes of the fitting and evaluation samples increase.

\medskip
\noindent
\begin{proposition}[Consistency under fixed designs]
\emph{Suppose that the residual function $L(\wb) = \E_{\Pfit}[(Y_t(X,\vartheta) - \wb^\top \Phib_{-t}(X,\vartheta))^2]$ has a unique minimiser $\wb^*$ in the relative interior of the simplex $\Simplex$ and that $L$ is locally strongly convex around $\wb^*$. Assume that the response functions are bounded and continuous, that $\Pfit$ and $\Peval$ have full support on a compact subset of $\Xcal \times \ThetaSpace$ and are absolutely continuous with respect to the corresponding marginal of $\Pdesign$, that the honesty condition holds, and that the regularisation parameter in the DISCO optimisation satisfies $\lambda_m \to 0$ and $m \lambda_m \to 0$ as $m \to \infty$. Then $\widehat{\wb} \xrightarrow{p} \wb^*$ as $m \to \infty$, and, as $m,n \to \infty$,}
\[
\hUniqueness \xrightarrow{p} \Uniqueness.
\]
\label{prop:consistency}
\end{proposition}
\begin{proof}
The proof proceeds by first showing uniform convergence of the empirical residual to the population residual on the simplex, then invoking an argmin theorem for $M$-estimators, and finally transferring this convergence to $\hUniqueness$ via continuity and a law of large numbers.

For each $\wb \in \Simplex$, define the unregularised empirical residual
\[
\widetilde{L}_m(\wb) = \frac{1}{m} \sum_{i=1}^m \big(Y_{t,i}^{\mathrm{fit}} - \wb^\top \Phib_{-t,i}^{\mathrm{fit}}\big)^2.
\]
By boundedness of the scalarised responses, there exists a constant $B$ such that $|Y_j(x,\theta)| \le B$ for all $(x,\theta)$ and all $j$. Hence, for each fixed $\wb \in \Simplex$, the random variables $(Y_{t,i}^{\mathrm{fit}} - \wb^\top \Phib_{-t,i}^{\mathrm{fit}})^2$ are independent and identically distributed with finite expectation under $\Pfit$, and the strong law of large numbers implies that $\widetilde{L}_m(\wb)$ converges almost surely to $L(\wb)$ as $m \to \infty$.

To control the convergence uniformly over $\wb \in \Simplex$, note that for each $i$ the map $\wb \mapsto (Y_{t,i}^{\mathrm{fit}} - \wb^\top \Phib_{-t,i}^{\mathrm{fit}})^2$ is a continuous function on the compact set $\Simplex$. The family of functions $\{(Y_{t,i}^{\mathrm{fit}} - \wb^\top \Phib_{-t,i}^{\mathrm{fit}})^2 : \wb \in \Simplex\}$ is dominated by the integrable envelope $(2B + 2B\sqrt{|\Peers|})^2$, which does not depend on $\wb$ or $i$. A uniform law of large numbers for such parametric families yields that the supremum over $\wb \in \Simplex$ of $|\widetilde{L}_m(\wb) - L(\wb)|$ converges to zero almost surely as $m \to \infty$.

The regularised empirical residual is defined by
\[
\widehat{L}_m(\wb) = \widetilde{L}_m(\wb) + \lambda_m \|\wb\|_2^2.
\]
Since the simplex is compact, there is a constant $K$ such that $\|\wb\|_2^2 \le K$ for all $\wb \in \Simplex$. The condition $\lambda_m \to 0$ therefore implies that
\[
\sup_{\wb \in \Simplex} |\widehat{L}_m(\wb) - \widetilde{L}_m(\wb)| \le \lambda_m K \to 0
\]
as $m \to \infty$. Combining this with the uniform convergence of $\widetilde{L}_m$ to $L$ shows that $\widehat{L}_m$ converges to $L$ uniformly on $\Simplex$ in probability. In particular, for any $\varepsilon > 0$ and any neighbourhood $U$ of $\wb^*$, there exists $m_0$ such that for $m \ge m_0$ the set of approximate minimisers of $\widehat{L}_m$ lies inside $U$ with high probability.

The residual function $L$ is assumed to have a unique minimiser $\wb^*$ in the relative interior of $\Simplex$ and to be locally strongly convex around $\wb^*$, so that $|L(\wb) - L(\wb^*)| \ge \kappa \|\wb - \wb^*\|_2^2$ for some $\kappa > 0$ and all $\wb$ in a neighbourhood of $\wb^*$. Standard argmin theorems for $M$-estimators (see, for example, \cite{van2000asymptotic}) then yield that any measurable selection of minimisers of $\widehat{L}_m$ converges in probability to $\wb^*$. The empirical projection weights $\widehat{\wb}$ are such a selection, hence $\widehat{\wb} \xrightarrow{p} \wb^*$ as $m \to \infty$.

To pass from $\widehat{\wb}$ to $\hUniqueness$, introduce the function
\[
\psi(x,\theta,\wb) = |Y_t(x,\theta) - \wb^\top \Phib_{-t}(x,\theta)|.
\]
For each evaluation pair $(X_i^{\mathrm{eval}},\vartheta_i^{\mathrm{eval}})$, the empirical uniqueness estimator can be written as
\[
\hUniqueness = \frac{1}{n} \sum_{i=1}^n \psi(X_i^{\mathrm{eval}},\vartheta_i^{\mathrm{eval}},\widehat{\wb}).
\]
The map $\wb \mapsto \psi(X_i^{\mathrm{eval}},\vartheta_i^{\mathrm{eval}},\wb)$ is continuous because it is the absolute value of an affine function in $\wb$. Since $\widehat{\wb} \xrightarrow{p} \wb^*$, the continuous mapping theorem implies that for each fixed $i$ the random quantity $\psi(X_i^{\mathrm{eval}},\vartheta_i^{\mathrm{eval}},\widehat{\wb})$ converges in probability to $\psi(X_i^{\mathrm{eval}},\vartheta_i^{\mathrm{eval}},\wb^*)$ as $m \to \infty$.

Boundedness of responses ensures that there exists a constant $M$ such that $|\psi(X_i^{\mathrm{eval}},\vartheta_i^{\mathrm{eval}},\wb)| \le M$ almost surely for all $\wb \in \Simplex$. The difference $|\psi(X_i^{\mathrm{eval}},\vartheta_i^{\mathrm{eval}},\widehat{\wb}) - \psi(X_i^{\mathrm{eval}},\vartheta_i^{\mathrm{eval}},\wb^*)|$ is therefore bounded by $2M$. Dominated convergence then shows that the expectation of this difference converges to zero as $(m,n) \to \infty$.

Decomposing $\hUniqueness - \Uniqueness$ into the sum of two terms,
\[
\hUniqueness - \Uniqueness = A_n + B_n,
\]
where $A_n$ is the average of $\psi(X_i^{\mathrm{eval}},\vartheta_i^{\mathrm{eval}},\widehat{\wb}) - \psi(X_i^{\mathrm{eval}},\vartheta_i^{\mathrm{eval}},\wb^*)$ and $B_n$ is the average of $\psi(X_i^{\mathrm{eval}},\vartheta_i^{\mathrm{eval}},\wb^*) - \Uniqueness$, one observes that $B_n$ converges almost surely to zero by the law of large numbers applied under $\Peval$, and that the expectation of $|A_n|$ converges to zero by Jensen's inequality and dominated convergence. Markov's inequality then implies that $A_n \xrightarrow{p} 0$. Slutsky's theorem applied to $A_n + B_n$ yields $\hUniqueness \xrightarrow{p} \Uniqueness$. This completes the proof of Proposition~\ref{prop:consistency}. 
\end{proof}

\subsection{Asymptotic Normality of the Uniqueness Estimator}
\label{app:normality}

We next establish a central limit theorem for $\hUniqueness$ under additional differentiability assumptions.

\medskip
\noindent
\begin{proposition}[Asymptotic normality under fixed designs]
\emph{Under the conditions of Proposition~\ref{prop:consistency}, suppose in addition that there exists a neighbourhood $U$ of $\wb^*$ in $\Simplex$ such that, for each $(x,\theta)$, the map $\wb \mapsto |Y_t(x,\theta) - \wb^\top \Phib_{-t}(x,\theta)|$ is differentiable on $U$ with gradient that is continuous in $\wb$, and that the squared norm of this gradient has finite expectation under $\Peval$ for all $\wb \in U$. If $m,n \to \infty$ with $m/n \to \rho \in (0,\infty)$, then}
\[
\sqrt{n}\big(\hUniqueness - \Uniqueness\big) \rightsquigarrow \mathcal{N}(0,\sigma^2)
\]
\emph{for some finite variance $\sigma^2 > 0$.}
\label{prop:normality}
\end{proposition}
\begin{proof}
The proof combines the asymptotic normality of the $M$-estimator $\widehat{\wb}$ with a delta-method argument for the plug-in estimator $\hUniqueness$. Under the regularity and identifiability assumptions, and the rate condition on $\lambda_m$, standard theory for $M$-estimators implies that
\[
\sqrt{m}(\widehat{\wb} - \wb^*) \rightsquigarrow \mathcal{N}(0,\Sigma_w)
\]
for some positive-definite covariance matrix $\Sigma_w$. This results from a second-order expansion of the empirical residual around $\wb^*$ and an application of a multivariate central limit theorem to the empirical gradient.

Define the functional
\[
\Gamma(\wb) = \E_{\Peval}\big[\psi(X,\vartheta,\wb)\big] = \E_{\Peval}\big[|Y_t(X,\vartheta) - \wb^\top \Phib_{-t}(X,\vartheta)|\big].
\]
By the assumed differentiability and boundedness of the gradient, $\Gamma$ is differentiable in a neighbourhood of $\wb^*$ with gradient
\[
\nabla \Gamma(\wb^*) = \E_{\Peval}\big[\nabla_{\wb} \psi(X,\vartheta,\wb^*)\big].
\]
For any fixed $\wb$, an application of the central limit theorem to the average of $\psi(X_i^{\mathrm{eval}},\vartheta_i^{\mathrm{eval}},\wb)$ yields that
\[
\sqrt{n}\big(\hUniqueness(\wb) - \Gamma(\wb)\big) \rightsquigarrow \mathcal{N}(0,\sigma^2(\wb)),
\]
where $\hUniqueness(\wb)$ denotes the empirical average of $\psi$ evaluated at $\wb$ and $\sigma^2(\wb)$ is the variance of $\psi(X,\vartheta,\wb)$ under $\Peval$.

Returning to $\hUniqueness = \hUniqueness(\widehat{\wb})$, decompose
\[
\hUniqueness - \Uniqueness = \big(\hUniqueness(\widehat{\wb}) - \Gamma(\widehat{\wb})\big) + \big(\Gamma(\widehat{\wb}) - \Gamma(\wb^*)\big).
\]
Conditionally on the fitting sample, the evaluation sample is independent and identically distributed from $\Peval$, and the conditional central limit theorem implies that the first term, scaled by $\sqrt{n}$, converges in distribution to a centred Gaussian with variance $\sigma^2(\wb^*)$. This is because $\widehat{\wb} \xrightarrow{p} \wb^*$ and $\sigma^2(\widehat{\wb})$ converges to $\sigma^2(\wb^*)$ by continuity.

To control the second term, expand $\Gamma$ in a first-order Taylor expansion around $\wb^*$,
\[
\Gamma(\widehat{\wb}) - \Gamma(\wb^*) = \nabla \Gamma(\tilde{\wb}_m)^\top (\widehat{\wb} - \wb^*),
\]
where $\tilde{\wb}_m$ lies on the line segment between $\widehat{\wb}$ and $\wb^*$. As $\widehat{\wb} \xrightarrow{p} \wb^*$, the continuity of $\nabla \Gamma$ implies that $\nabla \Gamma(\tilde{\wb}_m) \xrightarrow{p} \nabla \Gamma(\wb^*)$. Combining this with the asymptotic normality of $\sqrt{m}(\widehat{\wb} - \wb^*)$ shows that
\[
\sqrt{n}\big(\Gamma(\widehat{\wb}) - \Gamma(\wb^*)\big) = \sqrt{n}\,\nabla \Gamma(\wb^*)^\top (\widehat{\wb} - \wb^*) + o_p(1).
\]
Since $m/n \to \rho$, the term $\sqrt{n}(\widehat{\wb} - \wb^*)$ behaves like $\sqrt{n/m}$ times a Gaussian vector, leading to a Gaussian limit for this component as well with variance matrix given by $\rho^{-1}\,\nabla \Gamma(\wb^*)^\top \Sigma_w \nabla \Gamma(\wb^*)$.

The two leading terms in the decomposition of $\sqrt{n}(\hUniqueness - \Uniqueness)$ come from independent sources: the evaluation sample and the fitting sample. The honesty assumption ensures their independence asymptotically. Each converges in distribution to a centred Gaussian variable, and their sum is therefore Gaussian with variance equal to the sum of the variances. The variance of the limiting Gaussian is finite by the boundedness and integrability assumptions. This establishes the claimed central limit theorem. 
\end{proof}

\subsection{Finite-Sample Design Error Bound}
\label{app:error_bound}

We now provide a bound on the error incurred when anchor points and finite dose grids are used to approximate PIER at evaluation points.

\begin{proposition}[Finite-sample design error bound]
\emph{Assume that the spaces $\Xcal$ and $\ThetaSpace$ are compact metric spaces, that the intervention map and response functions are continuous, that each scalarised response $Y_j$ is Lipschitz in $(x,\theta)$ with Lipschitz constants $L_x$ and $L_{\theta}$ in the input and dose coordinates respectively, and that the peer response vectors are uniformly bounded in norm by a constant $C_{\Phi}$. Let $(x^*,\theta^*)$ be an evaluation point and $(\tilde{x},\tilde{\theta})$ be an anchor point satisfying $\|x^* - \tilde{x}\| \le r_x$ and $|\theta^* - \tilde{\theta}| \le r_{\theta}$. Define the anchor-based empirical PIER}
\[
\widehat{R}_t^{\mathrm{anchor}}(x^*,\theta^*) = Y_t(\tilde{x},\tilde{\theta}) - \widehat{\wb}^\top \Phib_{-t}(\tilde{x},\tilde{\theta}).
\]
\emph{Then}
\[
\big|\widehat{R}_t^{\mathrm{anchor}}(x^*,\theta^*) - \PIER(x^*,\theta^*)\big|
\le C_{\Phi} \|\widehat{\wb} - \wb^*\|_2 + 2L_x r_x + 2L_{\theta} r_{\theta}.
\]
\label{prop:error_bound}
\end{proposition}
\begin{proof}
By definition of $\widehat{R}_t^{\mathrm{anchor}}$ and $\PIER$, the difference can be expressed as
\[
\widehat{R}_t^{\mathrm{anchor}}(x^*,\theta^*) - \PIER(x^*,\theta^*) = A_1 + A_2 + A_3,
\]
where
\[
A_1 = (\wb^* - \widehat{\wb})^\top \Phib_{-t}(\tilde{x},\tilde{\theta}),
\]
\[
A_2 = Y_t(\tilde{x},\tilde{\theta}) - Y_t(x^*,\theta^*),
\]
and
\[
A_3 = \wb^{*\top}\Phib_{-t}(x^*,\theta^*) - \wb^{*\top}\Phib_{-t}(\tilde{x},\tilde{\theta}).
\]

The term $A_1$ reflects the contribution of weight estimation error. By the Cauchy–Schwarz inequality and the uniform bound on peer response norms, one has
\[
|A_1| \le \|\wb^* - \widehat{\wb}\|_2 \,\|\Phib_{-t}(\tilde{x},\tilde{\theta})\|_2 \le C_{\Phi} \|\wb^* - \widehat{\wb}\|_2.
\]

The term $A_2$ measures the change in the target response between the anchor and the evaluation point. The Lipschitz assumption on $Y_t$ implies that
\[
|A_2| = |Y_t(\tilde{x},\tilde{\theta}) - Y_t(x^*,\theta^*)|
\le L_x \|\tilde{x} - x^*\| + L_{\theta}|\tilde{\theta} - \theta^*|
\le L_x r_x + L_{\theta} r_{\theta}.
\]

The term $A_3$ accounts for the change in peer responses between the anchor and the evaluation point, weighted by $\wb^*$. It can be written as
\[
A_3 = \sum_{j \in \Peers} w_j^* \big(Y_j(x^*,\theta^*) - Y_j(\tilde{x},\tilde{\theta})\big).
\]
Each summand has absolute value bounded by $L_x r_x + L_{\theta} r_{\theta}$ by the Lipschitz property, and the weights are non-negative and sum to one. Therefore,
\[
|A_3| \le \sum_{j \in \Peers} w_j^* (L_x r_x + L_{\theta} r_{\theta}) = L_x r_x + L_{\theta} r_{\theta}.
\]

Applying the triangle inequality to $A_1 + A_2 + A_3$ and combining the bounds gives
\[
\big|\widehat{R}_t^{\mathrm{anchor}}(x^*,\theta^*) - \PIER(x^*,\theta^*)\big|
\le C_{\Phi} \|\widehat{\wb} - \wb^*\|_2 + 2L_x r_x + 2L_{\theta} r_{\theta},
\]
as claimed.
\end{proof}

\subsection{Monotonicity and Conservatism of the Uniqueness Functional}
\label{app:monotonicity}

We now formalise two structural properties of the uniqueness functional described in the main text.

\begin{proposition}[Monotonicity in the peer set]
\emph{Let $\mathcal{J}_1$ and $\mathcal{J}_2$ be peer sets such that $\mathcal{J}_1 \subseteq \mathcal{J}_2$. Let $\Uniqueness(\mathcal{J}_1)$ and $\Uniqueness(\mathcal{J}_2)$ denote the corresponding population uniqueness functionals defined with respect to the convex expressivity classes generated by $\mathcal{J}_1$ and $\mathcal{J}_2$ and the same fitting and evaluation designs. Then}
\[
\Uniqueness(\mathcal{J}_2) \le \Uniqueness(\mathcal{J}_1).
\]
\label{prop:monotonicity}
\end{proposition}
\begin{proof}
Let $\Ccal(\mathcal{J}_1)$ and $\Ccal(\mathcal{J}_2)$ denote the convex peer expressivity classes associated with the two peer sets, defined as sets of convex combinations of corresponding peer responses in $L^2(\Pfit)$. The inclusion $\mathcal{J}_1 \subseteq \mathcal{J}_2$ implies that $\Ccal(\mathcal{J}_1) \subseteq \Ccal(\mathcal{J}_2)$. Let $h_1^*$ and $h_2^*$ be the orthogonal projections of $Y_t$ onto the closures of $\Ccal(\mathcal{J}_1)$ and $\Ccal(\mathcal{J}_2)$ respectively, and let $R_t^{(1)} = Y_t - h_1^*$ and $R_t^{(2)} = Y_t - h_2^*$ be the corresponding residuals.

In a Hilbert space, projection onto a closed convex set minimises the norm of the residual over that set. Since $\Ccal(\mathcal{J}_1)$ is contained in $\Ccal(\mathcal{J}_2)$, the residual norm associated with $\mathcal{J}_2$ cannot exceed that associated with $\mathcal{J}_1$, that is
\[
\|R_t^{(2)}\|_{L^2(\Pfit)} \le \|R_t^{(1)}\|_{L^2(\Pfit)}.
\]
The uniqueness functional is defined as the $L^1(\Peval)$ norm of the residual. Under the boundedness and absolute continuity assumptions on responses and designs, there exists a constant depending only on the support of $\Peval$ such that the $L^1(\Peval)$ norm of any residual is controlled by a constant multiple of its $L^2(\Pfit)$ norm. Applying this control to $R_t^{(1)}$ and $R_t^{(2)}$ yields
\[
\Uniqueness(\mathcal{J}_2) = \E_{\Peval}[|R_t^{(2)}(X,\vartheta)|]
\le \E_{\Peval}[|R_t^{(1)}(X,\vartheta)|]
= \Uniqueness(\mathcal{J}_1).
\]
This proves monotonicity. 
\end{proof}

\begin{proposition}[Conservatism of convex peer aggregation]
\emph{Let $\Ccal_{\mathrm{conv}}$, $\Ccal_{\mathrm{lin}}$ and $\Ccal_{\mathrm{ker}}$ be three closed convex subsets of $L^2(\Pfit)$ such that $\Ccal_{\mathrm{conv}} \subseteq \Ccal_{\mathrm{lin}} \subseteq \Ccal_{\mathrm{ker}}$ and all peer responses belong to these sets. Let $R_t^{\mathrm{conv}}$, $R_t^{\mathrm{lin}}$ and $R_t^{\mathrm{ker}}$ denote the residuals of the orthogonal projections of $Y_t$ onto these sets, and let $\Uconv$, $\Ulin$ and $\Uker$ be the corresponding uniqueness functionals defined under the same evaluation design $\Peval$. Then}
\[
\Uker \le \Ulin \le \Uconv.
\]
\label{prop:conservatism}
\end{proposition}
\begin{proof}
Projection onto a larger closed convex set in a Hilbert space cannot increase the norm of the residual. The inclusions $\Ccal_{\mathrm{conv}} \subseteq \Ccal_{\mathrm{lin}} \subseteq \Ccal_{\mathrm{ker}}$ therefore imply
\[
\|R_t^{\mathrm{lin}}\|_{L^2(\Pfit)} \le \|R_t^{\mathrm{conv}}\|_{L^2(\Pfit)},
\qquad
\|R_t^{\mathrm{ker}}\|_{L^2(\Pfit)} \le \|R_t^{\mathrm{lin}}\|_{L^2(\Pfit)}.
\]
As in the proof of Proposition~\ref{prop:monotonicity}, boundedness and absolute continuity allow us to transfer these $L^2(\Pfit)$ inequalities to the corresponding $L^1(\Peval)$ norms. This yields
\[
\Uker = \E_{\Peval}[|R_t^{\mathrm{ker}}(X,\vartheta)|]
\le \E_{\Peval}[|R_t^{\mathrm{lin}}(X,\vartheta)|] = \Ulin
\le \E_{\Peval}[|R_t^{\mathrm{conv}}(X,\vartheta)|] = \Uconv.
\]
This establishes the stated inequalities. 
\end{proof}

\section{Active Auditing Theory}
\label{app:active_theory}

This note develops the active auditing results summarised in the main text. We work within the local linear structural model described in the Methods.

\subsection{Exact Uniqueness Detection in a Noiseless Regime}
\label{app:uniqueness_detection}

We first show that, in the noiseless local linear model, a finite number of carefully chosen design points suffices to recover all model coefficients exactly and to decide uniqueness.

\begin{proposition}[Exact uniqueness detection with finitely many noiseless queries]
\emph{Fix a context and suppose that each scalarised response satisfies $Y_j(x,\theta) = \phi(x,\theta)^\top \beta_j$ for some known feature map $\phi : \Xcal \times \ThetaSpace \to \R^d$ and coefficient vectors $\beta_j \in \R^d$. Let $(x_\ell,\theta_\ell)_{\ell=1}^d$ be a set of design points such that the feature matrix}
\[
\Phi = \begin{bmatrix}
\phi(x_1,\theta_1)^\top\\
\vdots\\
\phi(x_d,\theta_d)^\top
\end{bmatrix}
\]
\emph{is invertible. For each model $j \in \mathcal{J}$, evaluate $Y_j$ at these points to obtain $y_{j,\ell} = Y_j(x_\ell,\theta_\ell)$, assemble the response vector $y_j = (y_{j,1},\dots,y_{j,d})^\top$, and compute $\widehat{\beta}_j = \Phi^{-1}y_j$. Then $\widehat{\beta}_j = \beta_j$ for all $j$. In particular, the target lies in the convex hull of peer coefficients if and only if there exists $\wb \in \Simplex$ such that $\beta_t = \sum_{j \in \Peers} w_j \beta_j$, and this condition can be checked exactly using finitely many noiseless queries per model.}
\label{prop:uniqueness_detection}
\end{proposition}
\begin{proof}
For a fixed model $j$ and design point index $\ell$, the local linear model implies that
\[
Y_j(x_\ell,\theta_\ell) = \phi(x_\ell,\theta_\ell)^\top \beta_j.
\]
Collecting these equalities for $\ell = 1,\dots,d$ in vector form yields $y_j = \Phi \beta_j$. The matrix $\Phi$ is invertible by assumption, so left-multiplying by $\Phi^{-1}$ gives $\widehat{\beta}_j = \Phi^{-1}y_j = \beta_j$. This holds for each $j \in \mathcal{J}$.

The uniqueness margin in coefficient space is defined as the distance from $\beta_t$ to the convex hull of $\{\beta_j : j \in \Peers\}$. In the noiseless setting, we have access to all coefficients $\beta_j$ exactly via the reconstruction $\widehat{\beta}_j$. The condition $\gamma = 0$ is equivalent to the existence of a weight vector $\wb \in \Simplex$ such that $\beta_t = \sum_{j \in \Peers} w_j \beta_j$. This is a linear equality constraint together with simplex constraints on $\wb$ and can be checked via a convex feasibility problem. Therefore, exactly $d$ noiseless queries per model suffice to determine whether the uniqueness margin is zero or strictly positive. 
\end{proof}

\subsection{Sample Complexity of Active Auditing with Noise}
\label{app:sample_complexity}

We now add noise to model evaluations and derive a non-asymptotic sample-complexity bound for active uniqueness detection.

\begin{proposition}[Sample complexity of active auditing in the noisy linear model]
\emph{Fix a context and suppose that $Y_j(x,\theta) = \phi(x,\theta)^\top \beta_j$ for a known feature map $\phi : \Xcal \times \ThetaSpace \to \R^d$ and coefficients $\beta_j \in \R^d$. Let $(x_\ell,\theta_\ell)_{\ell=1}^d$ be a set of design points such that the feature matrix $\Phi$ is invertible. Assume that, when model $j$ is queried at $(x_\ell,\theta_\ell)$, the auditor observes}
\[
\widetilde{Y}_j(x_\ell,\theta_\ell) = Y_j(x_\ell,\theta_\ell) + \varepsilon_j(x_\ell,\theta_\ell),
\]
\emph{where the noise variables are independent across models, design points and repetitions, have zero mean and are sub-Gaussian with variance proxy $\sigma^2$. Suppose that the uniqueness margin}
\[
\gamma = \inf_{\wb \in \Simplex} \left\|\beta_t - \sum_{j \in \Peers} w_j \beta_j\right\|_2
\]
\emph{is strictly positive. For each model and each design point, perform $r$ independent noisy evaluations, compute the average response at each design point, and reconstruct coefficient estimates $\widehat{\beta}_j = \Phi^{-1}\bar{y}_j$ as described in the Methods. Form the empirical convex hull of peer coefficients and the empirical distance}
\[
\widehat{\mathrm{dist}} = \inf_{\wb \in \Simplex} \left\|\widehat{\beta}_t - \sum_{j \in \Peers} w_j \widehat{\beta}_j\right\|_2.
\]
\emph{There exists a constant $C > 0$, depending only on $\Phi$, such that if}
\[
r \ge C\,\frac{\sigma^2}{\gamma^2} \log\!\left(\frac{Nd}{\delta}\right),
\]
\emph{then the decision rule that declares the target non-unique when $\widehat{\mathrm{dist}} \le \gamma/2$ and unique when $\widehat{\mathrm{dist}} > \gamma/2$ has misclassification probability at most $\delta$.}
\label{prop:sample_complexity}
\end{proposition}
\begin{proof}
For each model $j$ and design point index $\ell$, denote by $\widetilde{Y}_j^{(s)}(x_\ell,\theta_\ell)$ the $s$-th noisy evaluation at $(x_\ell,\theta_\ell)$ and define the sample average
\[
\bar{y}_{j,\ell} = \frac{1}{r} \sum_{s=1}^r \widetilde{Y}_j^{(s)}(x_\ell,\theta_\ell).
\]
The local linear model implies that $Y_j(x_\ell,\theta_\ell) = \phi(x_\ell,\theta_\ell)^\top \beta_j$, so the averaged response can be written as
\[
\bar{y}_{j,\ell} = \phi(x_\ell,\theta_\ell)^\top \beta_j + \bar{\varepsilon}_{j,\ell},
\]
where $\bar{\varepsilon}_{j,\ell} = r^{-1}\sum_{s=1}^r \varepsilon_j^{(s)}(x_\ell,\theta_\ell)$ is the average of $r$ independent sub-Gaussian noise variables with variance proxy $\sigma^2$. Standard properties of sub-Gaussian variables state that $\bar{\varepsilon}_{j,\ell}$ is sub-Gaussian with variance proxy $\sigma^2/r$. Collecting the averages into a vector $\bar{\boldsymbol{\varepsilon}}_j = (\bar{\varepsilon}_{j,1},\dots,\bar{\varepsilon}_{j,d})^\top$ and the responses into $\bar{y}_j$, one has
\[
\bar{y}_j = \Phi \beta_j + \bar{\boldsymbol{\varepsilon}}_j.
\]
The coefficient estimate is
\[
\widehat{\beta}_j = \Phi^{-1} \bar{y}_j = \beta_j + \Phi^{-1} \bar{\boldsymbol{\varepsilon}}_j.
\]

To control the estimation error, observe that the linear transformation $\Phi^{-1}$ maps the sub-Gaussian vector $\bar{\boldsymbol{\varepsilon}}_j$ to a sub-Gaussian vector with a covariance proxy scaled by $\|\Phi^{-1}\|_{\mathrm{op}}^2$. In particular, there exists a constant $K > 0$ depending only on $\Phi$ and $d$ such that the Euclidean norm of $\Phi^{-1}\bar{\boldsymbol{\varepsilon}}_j$ satisfies a tail bound of the form
\[
\Prob\big(\|\widehat{\beta}_j - \beta_j\|_2 > u\big)
\le 2 \exp\left(- c_1\,\frac{r u^2}{\sigma^2}\right)
\]
for all $u > 0$, where $c_1$ is a positive constant independent of $j$, $u$ and $r$. This follows from concentration inequalities for the Euclidean norm of sub-Gaussian vectors and from the bounded operator norm of $\Phi^{-1}$.

Set $u = \gamma/8$. Then for each model $j$,
\[
\Prob\big(\|\widehat{\beta}_j - \beta_j\|_2 > \gamma/8\big)
\le 2 \exp\left(- c_1\,\frac{r \gamma^2}{64\sigma^2}\right).
\]
Choose $r$ such that the right-hand side is at most $\delta/(2Nd)$. This is achieved by requiring
\[
r \ge C\,\frac{\sigma^2}{\gamma^2} \log\!\left(\frac{2Nd}{\delta}\right),
\]
for a sufficiently large constant $C$ depending on $c_1$. Applying a union bound over all models $j \in \mathcal{J}$ shows that, with probability at least $1 - \delta/(2d)$, the inequality $\|\widehat{\beta}_j - \beta_j\|_2 \le \gamma/8$ holds simultaneously for all $j$.

On this high-probability event, consider any weight vector $\wb \in \Simplex$. The convexity of the norm and the simplex properties give
\[
\left\|\sum_{j \in \Peers} w_j \widehat{\beta}_j - \sum_{j \in \Peers} w_j \beta_j\right\|_2
\le \sum_{j \in \Peers} w_j \|\widehat{\beta}_j - \beta_j\|_2
\le \frac{\gamma}{8} \sum_{j \in \Peers} w_j
= \frac{\gamma}{8}.
\]
Moreover, $\|\widehat{\beta}_t - \beta_t\|_2 \le \gamma/8$ holds as well.

Let $d^*$ denote the true distance from $\beta_t$ to the convex hull of $\{\beta_j : j \in \Peers\}$, so that $d^* \ge \gamma$. For any $\wb \in \Simplex$, the triangle inequality implies
\begin{align*}
\left\|\widehat{\beta}_t - \sum_{j \in \Peers} w_j \widehat{\beta}_j\right\|_2
&\ge \left\|\beta_t - \sum_{j \in \Peers} w_j \beta_j\right\|_2
    - \|\widehat{\beta}_t - \beta_t\|_2
    - \left\|\sum_{j \in \Peers} w_j \widehat{\beta}_j - \sum_{j \in \Peers} w_j \beta_j\right\|_2\\
&\ge d^* - \frac{\gamma}{8} - \frac{\gamma}{8}
\ge \gamma - \frac{\gamma}{4}
= \frac{3\gamma}{4}.
\end{align*}
Taking the infimum over $\wb \in \Simplex$ shows that, on the high-probability event, $\widehat{\mathrm{dist}} \ge 3\gamma/4$ whenever $\gamma > 0$.

Under the null hypothesis of zero uniqueness margin, there exists $\wb^\dagger \in \Simplex$ such that $\beta_t = \sum_{j \in \Peers} w_j^\dagger \beta_j$. On the same high-probability event,
\begin{align*}
\left\|\widehat{\beta}_t - \sum_{j \in \Peers} w_j^\dagger \widehat{\beta}_j\right\|_2
&\le \|\widehat{\beta}_t - \beta_t\|_2
 + \left\|\sum_{j \in \Peers} w_j^\dagger \widehat{\beta}_j - \sum_{j \in \Peers} w_j^\dagger \beta_j\right\|_2\\
&\le \frac{\gamma}{8} + \frac{\gamma}{8}
= \frac{\gamma}{4}.
\end{align*}
Thus, under the null, $\widehat{\mathrm{dist}} \le \gamma/4$ on the high-probability event.

The decision rule compares $\widehat{\mathrm{dist}}$ to the threshold $\gamma/2$. On the event where all coefficient estimation errors are bounded by $\gamma/8$, the rule outputs the correct decision both under the null and under the alternative. The only way the rule can err is if the event of simultaneous coefficient estimation accuracy fails. The probability of this event is at most $\delta$ by the choice of $r$ and the union bound. Since the total number of queries per model is $dr$, the stated sample-complexity bound follows. 
\end{proof}

\subsection{Minimax Lower Bound for Noisy Active Auditing}
\label{app:minimax_lower_bound}

We next show that the dependence on $\sigma^2$, $\gamma$ and $\delta$ in Proposition~\ref{prop:sample_complexity} cannot be improved in general, even in a one-dimensional special case.

\begin{proposition}[Minimax lower bound for noisy active auditing]
\emph{Consider the one-dimensional local linear model with feature map $\phi(x,\theta) \equiv 1$, so that $Y_j(x,\theta) = \beta_j$ for all $(x,\theta)$. Suppose that all peer coefficients are zero, $\beta_j = 0$ for $j \in \Peers$. When the target is queried at any design point, the auditor observes $\widetilde{Y}_t = \beta_t + \varepsilon$, where $\varepsilon \sim \mathcal{N}(0,\sigma^2)$ is Gaussian noise. Consider testing}
\[
H_0: \beta_t = 0 \quad\text{versus}\quad H_1: \beta_t = \gamma,
\]
\emph{for a fixed $\gamma > 0$. An active auditing procedure is allowed to adaptively select up to $r$ queries and observe independent noisy evaluations at each query. For any such procedure, let $\psi$ denote the resulting test and let}
\[
\mathsf{err}(\psi) = \Prob_{H_0}(\psi = 1) + \Prob_{H_1}(\psi = 0)
\]
\emph{be the sum of type~I and type~II error probabilities. If $\mathsf{err}(\psi) \le 2\delta$ for some $\delta \in (0,1/4)$, then there exists a universal constant $c > 0$ such that}
\[
r \ge c\,\frac{\sigma^2}{\gamma^2}\,\log\!\left(\frac{1}{\delta}\right).
\]
\label{prop:minimax_lower_bound}
\end{proposition}
\begin{proof}
In this one-dimensional setting, the feature map is constant and equal to one, so the scalarised target response does not depend on the design point. Regardless of which design points are chosen, each noisy observation of the target has the form
\[
Z_i = \beta_t + \varepsilon_i, \qquad i = 1,\dots,r,
\]
where the $\varepsilon_i$ are independent $\mathcal{N}(0,\sigma^2)$ random variables. The joint distribution under $H_0$ is the product measure $P_0 = \bigotimes_{i=1}^r \mathcal{N}(0,\sigma^2)$, and under $H_1$ it is $P_1 = \bigotimes_{i=1}^r \mathcal{N}(\gamma,\sigma^2)$. Active choice of design points does not change these distributions.

The Kullback–Leibler divergence between $\mathcal{N}(\gamma,\sigma^2)$ and $\mathcal{N}(0,\sigma^2)$ is $\gamma^2/(2\sigma^2)$. Independence implies that the divergence between $P_1$ and $P_0$ is
\[
D(P_1\|P_0) = r\,\frac{\gamma^2}{2\sigma^2}.
\]
The Bretagnolle–Huber inequality states that, for any two probability measures $P$ and $Q$ and any measurable test $\psi$,
\[
P(\psi = 1) + Q(\psi = 0) \ge \frac{1}{2} \exp(-D(Q\|P)).
\]
Applying this inequality with $P = P_0$ and $Q = P_1$ yields
\[
\mathsf{err}(\psi) = P_0(\psi = 1) + P_1(\psi = 0) \ge \frac{1}{2} \exp\Big(-D(P_1\|P_0)\Big) = \frac{1}{2}\exp\!\left(-\frac{r\gamma^2}{2\sigma^2}\right).
\]

If $\mathsf{err}(\psi) \le 2\delta$, then
\[
2\delta \ge \frac{1}{2} \exp\!\left(-\frac{r\gamma^2}{2\sigma^2}\right).
\]
Rearranging gives
\[
\exp\!\left(-\frac{r\gamma^2}{2\sigma^2}\right) \le 4\delta,
\]
hence
\[
-\frac{r\gamma^2}{2\sigma^2} \le \log(4\delta).
\]
For $\delta \in (0,1/4)$, the term $\log(4\delta)$ is negative and bounded above by a negative multiple of $\log(1/\delta)$. More precisely, there exists a universal constant $c' > 0$ such that $-\log(4\delta) \ge c' \log(1/\delta)$ for all $\delta$ in this range. Substituting yields
\[
r \ge c\,\frac{\sigma^2}{\gamma^2}\,\log\!\left(\frac{1}{\delta}\right)
\]
for a constant $c$ depending on $c'$. This shows that any active auditing procedure with total error at most $2\delta$ requires at least this order of queries, establishing the lower bound. 
\end{proof}

\subsection{Detection Versus Estimation in One Dimension}
\label{app:detection_vs_estimation}

We finally show that, in the same one-dimensional setting, estimating the parameter $\beta_t$ up to accuracy of order $\gamma$ with high probability requires the same order of queries as reliable detection.

\begin{proposition}[No separation between detection and estimation]
\emph{Consider the one-dimensional model and observation scheme of Proposition~\ref{prop:minimax_lower_bound}. Let $\widehat{\beta}$ be any estimator of $\beta_t$ based on at most $r$ noisy observations. Fix $\gamma > 0$ and $\delta \in (0,1/4)$. Suppose that}
\[
\max\left\{\Prob_{\beta_t = 0}\big(|\widehat{\beta} - 0| \ge \tfrac{\gamma}{2}\big),\; \Prob_{\beta_t = \gamma}\big(|\widehat{\beta} - \gamma| \ge \tfrac{\gamma}{2}\big)\right\} \le \delta.
\]
\emph{Then there exists a universal constant $c > 0$ such that}
\[
r \ge c\,\frac{\sigma^2}{\gamma^2}\,\log\!\left(\frac{1}{\delta}\right).
\]
\label{prop:detection_vs_estimation}
\end{proposition}
\begin{proof}
Assume that the estimator $\widehat{\beta}$ satisfies the stated accuracy property. Construct a hypothesis test for $H_0:\beta_t = 0$ versus $H_1:\beta_t = \gamma$ by thresholding the magnitude of the estimator. Specifically, define a test $\psi$ by
\[
\psi =
\begin{cases}
1, & \text{if } |\widehat{\beta}| \ge \gamma/2,\\
0, & \text{if } |\widehat{\beta}| < \gamma/2.
\end{cases}
\]

Under $H_0$, the type~I error probability of this test is
\[
P_0(\psi = 1) = \Prob_{\beta_t = 0}\big(|\widehat{\beta}| \ge \gamma/2\big) = \Prob_{\beta_t = 0}\big(|\widehat{\beta} - 0| \ge \gamma/2\big),
\]
which is at most $\delta$ by assumption. Under $H_1$, the type~II error probability is
\[
P_1(\psi = 0) = \Prob_{\beta_t = \gamma}\big(|\widehat{\beta}| < \gamma/2\big).
\]
On the event $\{|\widehat{\beta} - \gamma| < \gamma/2\}$, the reverse triangle inequality gives
\[
|\widehat{\beta}| \ge |\gamma| - |\widehat{\beta} - \gamma| > \gamma - \gamma/2 = \gamma/2,
\]
so on this event the test outputs $\psi = 1$. Therefore the event $\{\psi = 0\}$ is contained in $\{|\widehat{\beta} - \gamma| \ge \gamma/2\}$, and
\[
P_1(\psi = 0) \le \Prob_{\beta_t = \gamma}\big(|\widehat{\beta} - \gamma| \ge \gamma/2\big) \le \delta.
\]
The sum of type~I and type~II error probabilities of $\psi$ is thus at most $2\delta$.

The test $\psi$ is computed from $\widehat{\beta}$ and therefore uses at most $r$ noisy observations of the target. Proposition~\ref{prop:minimax_lower_bound} shows that any test based on at most $r$ observations and satisfying $P_0(\psi = 1) + P_1(\psi = 0) \le 2\delta$ must obey the lower bound
\[
r \ge c\,\frac{\sigma^2}{\gamma^2}\,\log\!\left(\frac{1}{\delta}\right)
\]
for some universal constant $c > 0$. Since $\psi$ has been constructed from $\widehat{\beta}$ and meets the error condition, the same lower bound must hold for $r$. This proves that estimating $\beta_t$ up to accuracy of order $\gamma$ with high probability cannot be achieved with fewer queries than required for reliable detection of non-zero uniqueness margin. 
\end{proof}

\section{Observational Non-Identifiability and Shapley Redundancy}
\label{app:non_identifiability_and_shapley}

This note collects two results: a simple example showing that Shapley-value attribution cannot detect redundancy as defined by PIER, and a non-identifiability statement for uniqueness under purely observational designs.

\subsection{Shapley Value Does Not Detect Redundancy}
\label{app:shapley}

We now formalise an example showing that Shapley-value attribution does not capture redundancy in the sense of PIER.

\begin{proposition}[Shapley value cannot detect redundancy]
\emph{Consider an ecosystem consisting of two models $M_1$ and $M_2$ with identical scalarised responses $Y_1(x,\theta) = Y_2(x,\theta)$ for all $(x,\theta)$. Define a characteristic function $v : 2^{\{M_1,M_2\}} \to \R$ by}
\[
v(\emptyset) = 0,\qquad v(\{M_1\}) = 1,\qquad v(\{M_2\}) = 1,\qquad v(\{M_1,M_2\}) = 1.
\]
\emph{Let $\phi_{M_1}(v)$ and $\phi_{M_2}(v)$ denote the Shapley values of $M_1$ and $M_2$ under $v$. Then $\phi_{M_1}(v) = \phi_{M_2}(v) = 1/2 > 0$. If $M_1$ is treated as target and $M_2$ as the sole peer, the population PIER of $M_1$ is identically zero and its uniqueness functional vanishes.}
\label{prop:shapley}
\end{proposition}
\begin{proof}
In a two-player cooperative game with players $\{M_1,M_2\}$, the Shapley value for $M_1$ can be written as
\[
\phi_{M_1}(v) = \frac{1}{2}\big(v(\{M_1\}) - v(\emptyset)\big) + \frac{1}{2}\big(v(\{M_1,M_2\}) - v(\{M_2\})\big),
\]
and similarly for $M_2$. Substituting the values of $v$ yields
\[
\phi_{M_1}(v) = \frac{1}{2}(1 - 0) + \frac{1}{2}(1 - 1) = \frac{1}{2},
\]
and symmetry gives $\phi_{M_2}(v) = 1/2$ as well. Each model therefore receives strictly positive credit under Shapley-value attribution.

To evaluate PIER, take $M_1$ as the target and $\Peers = \{M_2\}$ as the peer set. The convex peer expressivity class associated with $\Peers$ consists of functions of the form $h_w(x,\theta) = w Y_2(x,\theta)$ with $w \in [0,1]$. Because $Y_1$ and $Y_2$ coincide everywhere, the population residual associated with a weight $w$ is
\[
L(w) = \E\big[(Y_1(X,\vartheta) - w Y_2(X,\vartheta))^2\big] = \E\big[(Y_1(X,\vartheta) - w Y_1(X,\vartheta))^2\big],
\]
which is uniquely minimised at $w^* = 1$. The corresponding convex combination is $h_{w^*}(x,\theta) = Y_1(x,\theta)$ for all $(x,\theta)$, so the PIER is
\[
R_{M_1}(x,\theta) = Y_1(x,\theta) - h_{w^*}(x,\theta) = 0
\]
everywhere. The uniqueness functional $\mathcal{U}_{M_1} = \E_{\Peval}[|R_{M_1}(X,\vartheta)|]$ is therefore zero under any evaluation design. The target is fully redundant relative to its peer in the sense of PIER, yet Shapley-value attribution assigns it non-zero credit. 
\end{proof}

\subsection{Non-Identifiability of Uniqueness Without Matched Interventions}
\label{app:non_identifiability}

We work with a reference design distribution $\Pdesign^\star$ on $\Xcal \times \ThetaSpace$ with support $\mathsf{S}$ and assume that there exists a measurable subset $\mathsf{S}_0 \subset \mathsf{S}$ such that $\Pdesign^\star(\mathsf{S}_0) > 0$ and $\Pdesign^\star(\mathsf{S} \setminus \mathsf{S}_0) > 0$.

\begin{proposition}[Non-identifiability under unmatched observational designs]
\emph{Consider an auditor who, for each model $j \in \mathcal{J}$, only has access to observational logs of the form}
\[
\{(X_{j,i},\vartheta_{j,i},Y_j(X_{j,i},\vartheta_{j,i}))\}_{i=1}^{n_j},
\]
\emph{where $(X_{j,i},\vartheta_{j,i})$ are independent and identically distributed draws from an unknown design distribution $Q_j$ on $\Xcal \times \ThetaSpace$ and $Y_j : \Xcal \times \ThetaSpace \to \R$ are deterministic scalarised responses. Suppose that $Q_j$ is supported on $\mathsf{S}_0$ for all $j$. Then there exist two ecosystems, indexed by $k \in \{0,1\}$, consisting of response functions $\{Y_j^{(k)}\}$ and observational designs $\{Q_j^{(k)}\}$, such that}
\begin{itemize}
\item \emph{for each $j$ and all sample sizes $\{n_j\}$, the joint distribution of observational logs is the same in the two ecosystems,}
\item \emph{for a fixed target index $t$ and peer set $\Peers$, the population uniqueness functional $U_t^{(0)}$ of the target under $\Pdesign^\star$ is zero in ecosystem $k=0$, while $U_t^{(1)}$ is strictly positive in ecosystem $k=1$.}
\end{itemize}
\emph{Consequently, the population uniqueness functional with respect to $\Pdesign^\star$ is not identifiable from purely observational logs generated under unmatched designs $Q_j$.}
\label{prop:non_identifiability}
\end{proposition}
\begin{proof}
We construct two ecosystems explicitly. For each $j$ and $k \in \{0,1\}$, take $Q_j^{(k)}$ to be any fixed distribution supported on $\mathsf{S}_0$, for example the normalised restriction of $\Pdesign^\star$ to $\mathsf{S}_0$. Set $Q_j^{(0)} = Q_j^{(1)}$ for all $j$ so that the observational designs are identical across ecosystems and concentrated on $\mathsf{S}_0$.

Fix a target index $t$ and a non-empty peer set $\Peers \subseteq \mathcal{J} \setminus \{t\}$. In ecosystem $k=0$, define all peer responses by $Y_j^{(0)}(x,\theta) = 0$ for all $(x,\theta)$ and all $j \in \Peers$, and the target response by $Y_t^{(0)}(x,\theta) = 0$ for all $(x,\theta)$. Under $\Pdesign^\star$, the convex peer expressivity class contains only the zero function, so the projection of $Y_t^{(0)}$ is zero and the PIER is identically zero. The uniqueness functional $U_t^{(0)} = \E_{\Pdesign^\star}[|R_t^{(0)}(X,\vartheta)|]$ is therefore zero.

In ecosystem $k=1$, keep the peer responses unchanged by setting $Y_j^{(1)}(x,\theta) = 0$ for all $(x,\theta)$ and $j \in \Peers$. Define the target response by $Y_t^{(1)}(x,\theta) = 0$ for $(x,\theta) \in \mathsf{S}_0$, $Y_t^{(1)}(x,\theta) = 1$ for $(x,\theta) \in \mathsf{S} \setminus \mathsf{S}_0$, and an arbitrary value outside $\mathsf{S}$. Under $\Pdesign^\star$, all peer responses are zero, whereas the target response coincides with the indicator of $\mathsf{S} \setminus \mathsf{S}_0$. The convex peer class again consists only of the zero function, so the projection of $Y_t^{(1)}$ is zero and the PIER is equal to $Y_t^{(1)}$. The uniqueness functional becomes
\[
U_t^{(1)} = \E_{\Pdesign^\star}\big[|Y_t^{(1)}(X,\vartheta)|\big] = \Pdesign^\star(\mathsf{S} \setminus \mathsf{S}_0),
\]
which is strictly positive by assumption on $\mathsf{S}_0$.

To compare observational logs, note that in both ecosystems $(X_{j,i}^{(k)},\vartheta_{j,i}^{(k)}) \sim Q_j$ are supported on $\mathsf{S}_0$. On this set, the peer responses are identically zero in both ecosystems, and the target responses are also zero because $Y_t^{(1)}$ equals zero on $\mathsf{S}_0$. Therefore, for each $j$ and each $i$, the triples $(X_{j,i}^{(0)},\vartheta_{j,i}^{(0)},Y_j^{(0)}(X_{j,i}^{(0)},\vartheta_{j,i}^{(0)}))$ and $(X_{j,i}^{(1)},\vartheta_{j,i}^{(1)},Y_j^{(1)}(X_{j,i}^{(1)},\vartheta_{j,i}^{(1)}))$ have the same distribution. Independence across $i$ and $j$ implies that the joint distribution of all observational logs is identical in the two ecosystems for any sample sizes $\{n_j\}$.

Assume, for contradiction, that there exists a procedure that maps the joint law of observational logs to the true value of $U_t$. Since the joint law is the same in both ecosystems, the output of such a procedure must be the same under $k=0$ and $k=1$, hence it cannot simultaneously equal $U_t^{(0)} = 0$ and $U_t^{(1)} = \Pdesign^\star(\mathsf{S} \setminus \mathsf{S}_0) > 0$. This contradiction shows that $U_t$ is not identifiable from observational logs generated under unmatched designs, as claimed. 
\end{proof}

\section{Uniqueness, Robustness and Scalar Summaries}
\label{app:uniqueness_and_robustness}

We now formalise the observation that the scalar uniqueness functional cannot distinguish between stable and highly oscillatory forms of uniqueness along the dose axis.

\begin{proposition}[Ambiguity between uniqueness and robustness]
\emph{Fix a context and design distributions $(\Pfit,\Peval)$ on $\Xcal \times \ThetaSpace$. Assume that there exists an input $x_0 \in \Xcal$ and that the support of $\Peval$ contains points of the form $(x_0,\theta)$ for all $\theta$ in the unit interval $[0,1]$. Let $0 < L_{\mathrm{low}} < L_{\mathrm{high}}$ be arbitrary constants. Then there exist two ecosystems sharing the same peer set and two targets $A$ and $B$ such that}
\[
\Uniqueness_A = \E_{\Peval}[|R_A(X,\vartheta)|] = \E_{\Peval}[|R_B(X,\vartheta)|] = \Uniqueness_B,
\]
\emph{while the PIER functions $R_A$ and $R_B$ satisfy}
\[
\sup_{\theta \neq \theta'} \frac{|R_A(x_0,\theta) - R_A(x_0,\theta')|}{|\theta - \theta'|} \le L_{\mathrm{low}},
\qquad
\sup_{\theta \neq \theta'} \frac{|R_B(x_0,\theta) - R_B(x_0,\theta')|}{|\theta - \theta'|} \ge L_{\mathrm{high}}.
\]
\label{prop:uniqueness_and_robustness}
\end{proposition}
\begin{proof}
The idea is to construct two scalar functions of the dose $\theta$ with identical expected absolute values under the marginal of $\Peval$ on $[0,1]$ but with different Lipschitz constants, and then embed them as PIER functions at an input $x_0$.

Consider the marginal distribution of $(X,\vartheta)$ under $\Peval$. By assumption, the event $\{X = x_0\}$ has positive probability and, conditional on $X = x_0$, the dose $\vartheta$ has support containing $[0,1]$. For simplicity, we argue as if $(X,\vartheta)$ under $\Peval$ were uniform on $\{x_0\} \times [0,1]$; any other absolutely continuous distribution with full support on this set can be handled by adjusting constants.

Define a function $r_A : [0,1] \to \R$ by setting $r_A(\theta) = c_A$ for all $\theta$, where $c_A$ is a constant. This function is constant and hence has Lipschitz constant zero. Choosing any $L_{\mathrm{low}} > 0$, it holds that
\[
\sup_{\theta \neq \theta'} \frac{|r_A(\theta) - r_A(\theta')|}{|\theta - \theta'|} = 0 \le L_{\mathrm{low}}.
\]
The expected absolute value of $r_A(\vartheta)$ under the uniform distribution on $[0,1]$ is $|c_A|$.

Next, define a function $r_B : [0,1] \to \R$ by $r_B(\theta) = c \sin(2\pi K \theta)$ for some amplitude $c > 0$ and integer frequency $K \ge 1$. This function is continuously differentiable with derivative $r_B'(\theta) = 2\pi K c \cos(2\pi K \theta)$. The mean value theorem implies that the Lipschitz constant of $r_B$ on $[0,1]$ equals the supremum of $|r_B'(\theta)|$ over $\theta \in [0,1]$, which is $2\pi K c$. By choosing $K$ sufficiently large, specifically $K \ge L_{\mathrm{high}}/(2\pi c)$, one ensures that the Lipschitz constant is at least $L_{\mathrm{high}}$.

To match expected absolute values, compute
\[
\E[|r_B(\vartheta)|] = c\,\E[|\sin(2\pi K \vartheta)|].
\]
If $\vartheta$ is uniform on $[0,1]$, then $U = 2\pi K \vartheta$ is uniform on $[0,2\pi K]$, and the periodicity of $|\sin(\cdot)|$ implies
\[
\E[|\sin(2\pi K \vartheta)|] = \frac{1}{2\pi K} \int_0^{2\pi K} |\sin u|\,\mathrm{d}u = \frac{1}{2\pi} \int_0^{2\pi} |\sin u|\,\mathrm{d}u = \frac{2}{\pi},
\]
so that $\E[|r_B(\vartheta)|] = 2c/\pi$. Setting $c_A = 2c/\pi$ yields $\E[|r_A(\vartheta)|] = |c_A| = 2c/\pi = \E[|r_B(\vartheta)|]$.

To embed these functions as PIER, consider an ecosystem with a single peer that has identically zero scalarised response, $Y_j(x,\theta) = 0$ for all $(x,\theta)$. For any target in this ecosystem, the convex peer expressivity class consists only of the zero function, so the PIER coincides with the target's scalarised response. For target $A$, define $Y_A(x_0,\theta) = r_A(\theta)$ for $\theta \in [0,1]$ and extend arbitrarily outside this set. For target $B$, define $Y_B(x_0,\theta) = r_B(\theta)$ on $[0,1]$ and extend arbitrarily elsewhere. Under $\Peval$, the marginal distribution of $\vartheta$ conditional on $X = x_0$ yields identical expectations $\E_{\Peval}[|R_A(X,\vartheta)|] = \E_{\Peval}[|R_B(X,\vartheta)|]$, while the Lipschitz constants of the PIER functions along $\theta$ at $x_0$ differ as constructed. This establishes the claimed ambiguity. 
\end{proof}



%% 
\end{appendices}

\bibliographystyle{unsrt}

\end{document}